\documentclass[10pt,twocolumn,letterpaper]{article}

\usepackage[utf8]{inputenc}
\usepackage{cvpr}
\usepackage{times}
\usepackage{epsfig}
\usepackage{graphicx}
\graphicspath{{main-fig/}}
\usepackage{algorithm,algorithmic}
\usepackage{amsmath,amssymb,amsthm}
\usepackage[table,xcdraw]{xcolor}

\usepackage[pagebackref=true,breaklinks=true,colorlinks,bookmarks=false]{hyperref}
\usepackage{myshortcuts}
\usepackage{subcaption}
\usepackage{booktabs} 
\usepackage{multirow}
\usepackage{placeins}
\usepackage[final]{microtype}
\usepackage[inline]{enumitem}
\usepackage{printlen}

\usepackage{chngcntr}

\setlist{itemsep=0pt,topsep=1ex,partopsep=1ex,parsep=1ex}

\usepackage{pdfpages}
\usepackage{stfloats}
\usepackage{tablefootnote}

\makeatletter
\def\thmhead@plain#1#2#3{%
  \thmname{#1}\thmnumber{\@ifnotempty{#1}{ }\@upn{#2}}%
  \thmnote{ {\the\thm@notefont#3}}}
\let\thmhead\thmhead@plain
\makeatother

\makeatletter
\renewcommand{\paragraph}{%
  \@startsection{paragraph}{4}%
  {\z@}{0.5ex \@plus 0.1ex \@minus 0.3ex}{-1em}%
  {\normalfont\normalsize\bfseries}%
}
\makeatother

\newtheorem{definition}{Definition}
\newtheorem{proposition}{Proposition}

\newcommand{\todo}[1]{\textcolor{red}{\textbf{Fix:} \textit{#1}}}
\cvprfinalcopy 

\def\cvprPaperID{2291}

\newif\ifarxiv
\arxivtrue

\ifarxiv
  \usepackage{fancyhdr}
\else
  \ifcvprfinal
  \fi
\fi

\let\OLDthebibliography\thebibliography
\renewcommand\thebibliography[1]{
  \OLDthebibliography{#1}
  \setlength{\parskip}{0.1ex}
  \setlength{\itemsep}{0.1ex plus 0.3ex}
}

\makeatletter
\def\ps@myheadings{%
    \let\@oddfoot\@empty\let\@evenfoot\@empty
    \def\@evenhead{\thepage\hfil\slshape\leftmark}%
    \def\@oddhead{{\slshape\rightmark}\hfil\thepage}%
    \let\@mkboth\@gobbletwo
    \let\sectionmark\@gobble
    \let\subsectionmark\@gobble
    }
  \if@titlepage
  \renewcommand\maketitle{\begin{titlepage}%
  \let\footnotesize\small
  \let\footnoterule\relax
  \let \footnote \thanks
  \null\vfil
  \vskip 60\p@
  \begin{center}%
    {\LARGE \@title \par}%
    \vskip 3em%
    {\large
     \lineskip .75em%
      \begin{tabular}[t]{c}%
        \@author
      \end{tabular}\par}%
      \vskip 1.5em%
    {\large \@date \par}
  \end{center}\par
  \@thanks
  \vfil\null
  \end{titlepage}%
  \setcounter{footnote}{0}%
}
\else
\renewcommand\maketitle{\par
  \begingroup
    \renewcommand\thefootnote{\@fnsymbol\c@footnote}%
    \def\@makefnmark{\rlap{\@textsuperscript{\normalfont\@thefnmark}}}%
    \long\def\@makefntext##1{\parindent 1em\noindent
            \hb@xt@1.8em{%
                \hss\@textsuperscript{\normalfont\@thefnmark}}##1}%
    \if@twocolumn
      \ifnum \col@number=\@ne
        \@maketitle
      \else
        \twocolumn[\@maketitle]%
      \fi
    \else
      \newpage
      \global\@topnum\z@   
      \@maketitle
    \fi
    \@thanks
  \endgroup
  \setcounter{footnote}{0}%
}
\makeatother

\begin{document}

\title{The Lovász-Softmax loss: A tractable surrogate for the optimization of the intersection-over-union measure in neural networks}

\author{Maxim Berman \quad Amal Rannen Triki \quad Matthew B. Blaschko\\
	Dept.\ ESAT, Center for Processing Speech and Images\\
 	KU Leuven, Belgium\\
{\tt\small \{maxim.berman,amal.rannen,matthew.blaschko\}@esat.kuleuven.be} \\
}

\maketitle

\begin{abstract}
The Jaccard index, also referred to as the intersection-over-union score, is commonly employed in the evaluation of image segmentation results given its perceptual qualities, scale invariance -- which lends appropriate relevance to small objects, and appropriate counting of false negatives, in comparison to per-pixel losses. We present a method for direct optimization of the mean intersection-over-union loss in neural networks, in the context of semantic image segmentation, based on the convex Lovász extension of submodular losses. The loss is shown to perform better with respect to the Jaccard index measure than the traditionally used cross-entropy loss. We show quantitative and qualitative differences between optimizing the Jaccard index per image versus optimizing the Jaccard index taken over an entire dataset. We evaluate the impact of our method in a semantic segmentation pipeline and show substantially improved intersection-over-union segmentation scores on the Pascal VOC and Cityscapes datasets using state-of-the-art deep learning segmentation architectures.
\end{abstract}

\section{Introduction}

We consider the task of semantic image segmentation, where each pixel $i$ of a given image has to be classified into an object class $c \in \mathcal{C}$. 
Most of the deep network based segmentation methods rely on logistic regression, optimizing the cross-entropy loss \cite{Goodfellow-et-al-2016}
\begin{equation}\label{xloss}
\operatorname{loss}(\vec{f}) = - \frac{1}{p}\sum_{i=1}^p \log f_i(y_i^*) ,
\end{equation}
with $p$ the number of pixels in the image or minibatch considered, $y_i^*\in\mathcal{C}$ the ground truth class of pixel $i$, $f_i(y_i^*)$ the network probability estimate of the ground truth probability of pixel~$i$, and $\vec{f}$ a vector of all network outputs $f_i(c)$. 
This supposes that the unnormalized scores $F_i(c)$ of the network have been mapped to probabilities through a \emph{softmax} unit
\begin{equation}\label{softmax}
f_i(c) = \frac{e^{F_i(c)}}{\sum_{c' \in \mathcal{C}} e^{F_i(c')}}\quad\forall{i \in [1, p]}, \forall{c \in \mathcal{C}}.
\end{equation}
Loss~\eqref{xloss} generalizes the logistic loss
and leads to smooth optimization.
During testing, the decision function commonly used consists in picking the class of maximum score: the predicted class for a given pixel $i$ is
$\label{eq:argmax-decision}
\tilde{y}_i = \argmax_{c\in\mathcal{C}} F_i(c).
$

The measure of the cross-entropy loss on a validation set is often a poor indicator of the quality of the segmentation. A better performance measure commonly used for evaluating segmentation masks is the Jaccard index, also called the intersection-over-union (IoU) score. 
Given a vector of ground truth labels $\gt$ and a vector of predicted labels $\pred$, the Jaccard index of class $c$ is defined as~\cite{Jaccard1901}
\begin{equation}\label{eq:JaccardIndex}
J_c(\gt, \pred) = \frac{|\{\gt = c\} \cap\{\pred = c\}|}{|\{\gt = c\} \cup \{\pred = c\}|},
\end{equation}
which gives the ratio in $[0, 1]$ of the intersection between the ground truth mask and the evaluated mask over their union, with the convention that $0/0=1$. 
A corresponding loss function to be employed in empirical risk minimization is
\begin{equation}\label{eq:JaccardLoss}
\J(\gt, \pred) = 1 - J_c(\gt, \pred). 
\end{equation}
For multilabel datasets, the Jaccard index is commonly averaged across classes, yielding the mean IoU (mIoU). 

We develop here a method for optimizing the performance of a discriminatively trained segmentation system with respect to the Jaccard index. We show that a piecewise linear convex surrogate to the Jaccard loss based on the Lovász extension of submodular set functions yields a consistent improvement of predicted segmentation masks as measured by the Jaccard index.

Although the Jaccard index is often computed globally, over every pixel of the evaluated segmentation dataset~\cite{pascal-voc-2012}, it can also be computed independently for each image. Using the per-image Jaccard index is known to have better perceptual accuracy by reducing the bias towards large instances of the object classes in the dataset~\cite{csurka2013good}.  Due to these favorable properties, and the empirical risk minimization principle of optimizing the loss of interest at training time~\cite{Vapnik1995}, optimization of the Jaccard loss during training has been frequently considered in the literature.  However, in contrast to the present work, existing methods all have significant shortcomings that do not allow plug-and-play application to a wide range of learning architectures.

\cite{Nowozin_2014_CVPR} provides a Bayesian framework for optimization of the Jaccard index. The author proposes an approximate algorithm using parametric linear programming to optimize a statistical approximation to the objective. \cite{Ahmed_2015_ICCV} optimize IoU by selecting among a few candidate segmentations, instead of directly optimizing the model with respect to the loss. 
\cite{Blaschko2008d} optimize the Jaccard loss in a structured output SVM, but are only able to do so with a branch-and-bound optimization over bounding boxes and not full segmentations.

Alternative approaches train binary classifiers, but on data that are sampled to capture high Jaccard index.
\cite{Bourdev2010,Hariharan2014} use IoU and related overlap measures to define training sets for binary classifiers in a complex multi-stage training. Such sampling-based approaches clearly induce suboptimality in the empirical risk approximation and do not lend themselves to convenient modular application in a deep learning setting.

Still other recent high-impact research has highlighted the need for optimization of the Jaccard index, but resort to binary training as a proxy, presumably for lack of a convenient and flexible method of directly optimizing the loss of interest.
\cite{Long_2015_CVPR} train with logistic loss and test with the Jaccard index.  The paper introducing the highly influential OverFeat network specifically addresses the shortcoming in the discussion section \cite{DBLP:journals/corr/SermanetEZMFL13}:
``We are
using $\ell_2$
loss, rather than directly optimizing the intersection-over-union (IoU) criterion on which
performance is measured.  Swapping the loss to this should be
possible....''  However, this is left to future work.  In this paper, we develop the necessary plug-and-play loss layer to enable flexible direct minimization of the Jaccard loss in a deep learning setting, while demonstrating its applicability for training state-of-the-art image segmentation networks.

Our approach is based on the recent development of general strategies for generating convex surrogates to submodular loss functions, including the Lovász hinge \cite{yu:hal-01151823}. 
Based on the result that the Jaccard loss is submodular, this strategy is directly applicable. 
We moreover generalize this approach to a multiclass setting by considering a regression-based variant, using a softmax activation layer to naturally map network probability estimates to the Lovász extension of the Jaccard loss. 
In this work, we
\begin{enumerate*}[label=(\roman*)]
\item apply the Lovász hinge with Jaccard loss to the problem of binary image segmentation (Sec.~\ref{sec:binarySurrogate}),
\item propose a surrogate for the multi-class setting, the Lovász-Softmax loss (Sec.~\ref{sec:multiclassSurrogate}),
\item design a batch-based IoU surrogate that acts as an efficient proxy to the dataset IoU measure (Sec.~\ref{sec:differentMeasures}),
\item analyze and compare the properties of different IoU-based measures, and
\item demonstrate a substantial and consistent improvement in performance measured by the Jaccard index in state-of-the-art deep learning based segmentation systems.
\end{enumerate*}

\section{Optimization surrogates for submodular loss functions} 
In order to optimize the Jaccard index in a continuous optimization framework, we consider smooth extensions of this discrete loss. The extensions are based on submodular analysis of set functions, where the set function maps from a set of mispredictions to the set of real numbers \cite[Equation~(6)]{yu:hal-01151823}.

For a segmentation output $\pred$ and ground truth $\gt$, we define the set of mispredicted pixels for class $c$ as
\begin{equation}\label{eq:mispred}
\mathbf{M}_c(\gt, \pred) = \{\gt = c, \pred \neq c\} \cup 
\{\gt \neq c, \pred = c\}.
\end{equation}
For a fixed ground truth $\gt$, the Jaccard loss in Eq.~\eqref{eq:JaccardLoss} can be rewritten as a function of the set of mispredictions 
\begin{equation}\label{eq:JaccardLossMispred}
  \J \colon \mathbf{M}_c \in \{0, 1\}^p \mapsto 
  \frac{|\mathbf{M}_c|}
     {|\{\gt = c\} \cup \mathbf{M}_c|}.
\end{equation}
Note that for ease of notation, we naturally identify subsets of pixels with their indicator vector in the discrete hypercube~$\{0, 1\}^p$. 

In a continuous optimization setting, we want to assign a loss to any vector of errors $\vec{m} \in \mathbb{R}_+^p$, and not only to discrete vectors of mispredictions in $\{0, 1\}^p$. 
A natural candidate for this loss is the convex closure of function~\eqref{eq:JaccardLossMispred} in $\mathbb{R}^p$. 
In general, computing the convex closure of set functions is NP-hard. 
However, the Jaccard set function~\eqref{eq:JaccardLossMispred} has been shown to be submodular~\cite[Proposition~11]{Yu2015b}.
\begin{definition}[{\cite{fujishige2005submodular}}]
A set function $\setF: \{0, 1\}^p \rightarrow \mathbb{R}$ is \emph{submodular} if for all $\mathbf{A}, \mathbf{B} \in \{0, 1\}^p$
\begin{equation}
\setF(\mathbf{A}) + \setF(\mathbf{B}) \geq \setF(\mathbf{A} \cup \mathbf{B}) + \setF(\mathbf{A} \cap \mathbf{B}).
\end{equation}
\end{definition}
The convex closure of submodular set functions is~\emph{tight} and~\emph{computable in polynomial time}~\cite{lovasz1983submodular}; it corresponds to its Lovász extension.
\begin{definition}[{\cite[Def. 3.1]{bach2013learning}}]
The \emph{Lovász extension} of a set function $\setF\colon \{0, 1\}^p \to \mathbb{R}$ such that $\setF(\vec{0}
) = 0$ is defined by
\begin{equation}\label{eq:lovasz_extension}
\ext{\setF}\colon \vec{m} \in \mathbb{R}^p \mapsto \sum_{i=1}^{p} m_i \, g_i(\vec{m})\end{equation}
\begin{equation}\label{eq:lovasz_derivative}
\text{with}\,\,g_i(\vec{m}) =
\setF(\{\perm_1, \ldots, \perm_i\})- \setF(\{\perm_1, \ldots, \perm_{i-1}\}),
\end{equation}
$\vec{\pi}$ being a permutation ordering the components of $\vec{m}$ in decreasing order, i.e. $x_{\perm_1} \geq x_{\perm_2} \ldots \geq x_{\perm_p}$.
\end{definition}

Let $\Delta$ be a set function encoding a submodular loss such as the Jaccard loss defined in Equation~\eqref{eq:JaccardLossMispred}. By submodularity $\ext{\Delta}$ is the tight convex closure of $\Delta$~\cite{lovasz1983submodular}. $\ext{\Delta}$ is piecewise linear and interpolates the values of $\Delta$ in $\mathbb{R}^p \setminus \{0, 1\}^p$, while having the same values as $\Delta$ on $\{0, 1\}^p$, i.e. on any set of mispredictions (Equation~\eqref{eq:mispred}).
Intuitively, if $\vec{m}$ is a vector of all pixel errors, $\ext{\Delta}(\vec{m})$ is a sum weighting these errors according to the interpolated discrete loss. 
By its convexity and continuity, $\ext{\Delta}$ is a natural surrogate for the minimization of $\Delta$ with first-order continuous optimization, such as in neural networks. 
The elementary operations involved to compute $\ext{\Delta}$ (\emph{sort}, \emph{dot~product}, \ldots) are differentiable and implemented on GPU in current deep learning frameworks. 
The vector $\vec{g}(\vec{m})$ of which the components are defined in Equation~\eqref{eq:lovasz_derivative} directly corresponds to the derivative of $\ext{\Delta}$ with respect to $\vec{m}$.

In the following, we consider two different settings in which we construct surrogate losses by using the Lovász extension and specifying the vector of errors $\vec{m}$ that we use:
\begin{enumerate}
\item The foreground-background segmentation problem, which leads to the Lovász hinge, as described in~\cite{Yu2015b};
\item The multiclass segmentation problem, which leads to the Lovász-Softmax loss, incorporating the softmax operation in the Lovász extension.
\end{enumerate}

\subsection{Foreground-background segmentation}\label{sec:binarySurrogate}

\begin{figure}[ht]
  \centering
  \begin{subfigure}{0.45\linewidth}    \centering\includegraphics[height=1.75cm]{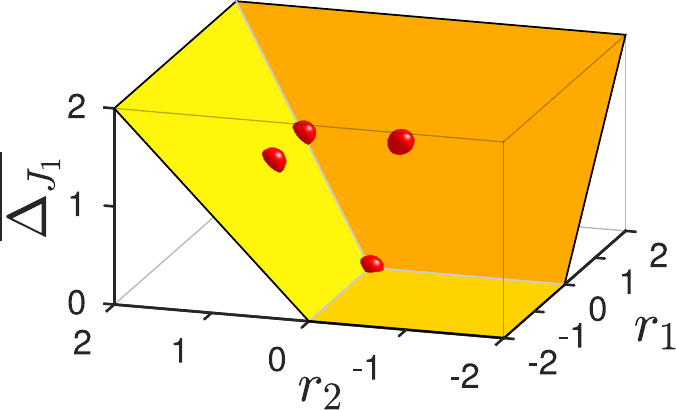}
    \caption{$\GT = [-1, -1]$}
  \end{subfigure}%
\begin{subfigure}{0.45\linewidth}    \centering\includegraphics[height=1.75cm]{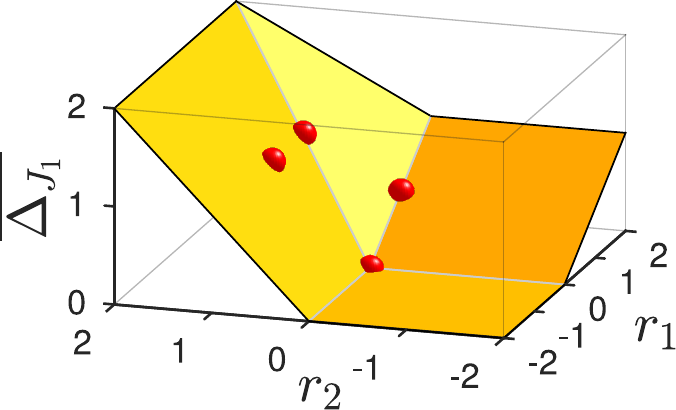}
    \caption{$\GT = [-1, 1]$}
  \end{subfigure}
 \begin{subfigure}{0.45\linewidth}    \centering\includegraphics[height=1.75cm]{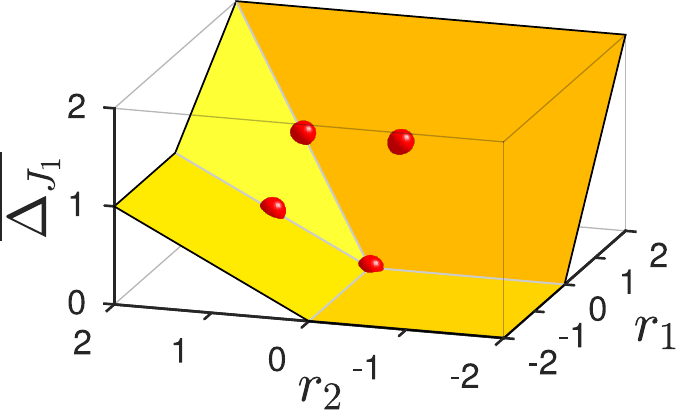}
    \caption{$\GT = [1, -1]$}
  \end{subfigure}%
 \begin{subfigure}{0.45\linewidth}    \centering\includegraphics[height=1.75cm]{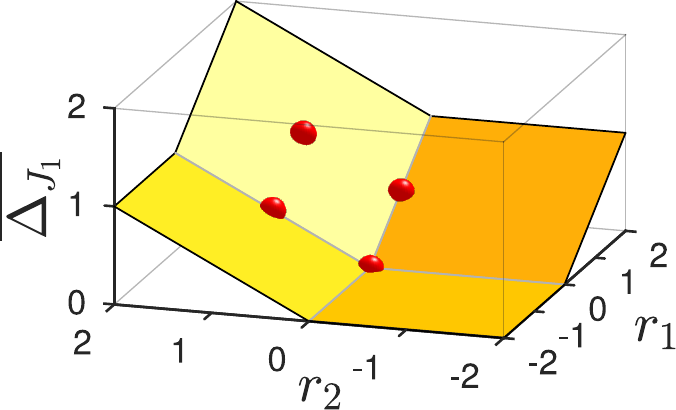}
    \caption{$\GT = [1, 1]$}
  \end{subfigure}%
\caption{Lovász hinge in the case of two pixel predictions for the four possible ground truths $\GT$, as a function of the relative margins $r_i = 1 - F_i(\vec{x})\,y^*_i$ for $i = 1, 2$. The red dots indicate the values of the discrete Jaccard index.\label{fig:lovasz_hinge}
}
\end{figure}
In the binary case, we consider the optimization of the Jaccard index for the foreground class $\Jf$. 
We use a max-margin classifier: for an image $\mathbf{x}$, we define
\begin{itemize}
\item $y^*_i \in \{-1, 1\}$ the ground truth label of pixel $i$,
\item $F_i(\mathbf{x})$ the $i$-th element of the output scores $\vec{F}$ of the model, such that the predicted label $\tilde{y}_i= \operatorname{sign}(F_i(\vec{x}))$,
\item $\marg_i = \max(1 - F_i(\vec{x})\,y^*_i, 0) $ the hinge loss associated with the prediction of pixel $i$.
\end{itemize}

In this setting, the vector of hinge losses $\margs \in \mathbf{R}^{+}$ is the vectors of errors discussed before. 
With $\ext{\Jf}$ the Lovász extension to $\Jf$, the resulting loss surrogate
\begin{equation}\label{eq:lovasz_hinge}
\mathrm{loss}(\vec{F}) = \ext{\Jf}(\margs(\vec{F}))
\end{equation}
is the Lovász hinge applied to the Jaccard loss, as described in~\cite{yu:hal-01151823}. 
It is piecewise linear in the output scores $\vec{F}$ as a composition of piecewise linear functions. 
Moreover, by choice of the hinge loss for the vector $\vec{m}$, the Lovász hinge reduces to the standard hinge loss~\cite{vapnik1998statistical} in the case of a single prediction, or when using the Hamming distance instead of the Jaccard loss as a basis for the construction.
Figure~\ref{fig:lovasz_hinge} illustrates the extension of the Jaccard loss in the case of the prediction of two pixels, illustrating the convexity and the tightness of the surrogate.

\subsection{Multiclass semantic segmentation} \label{sec:multiclassSurrogate}

\begin{figure}[ht]
  \centering
  \begin{subfigure}{0.45\linewidth}    \centering\includegraphics[height=2cm]{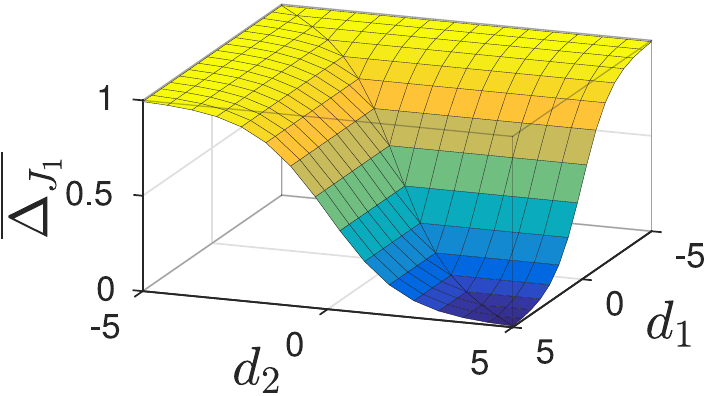}
    \caption{$\GT = [-1, -1]$}
  \end{subfigure}
\begin{subfigure}{0.45\linewidth}    \centering\includegraphics[height=2cm]{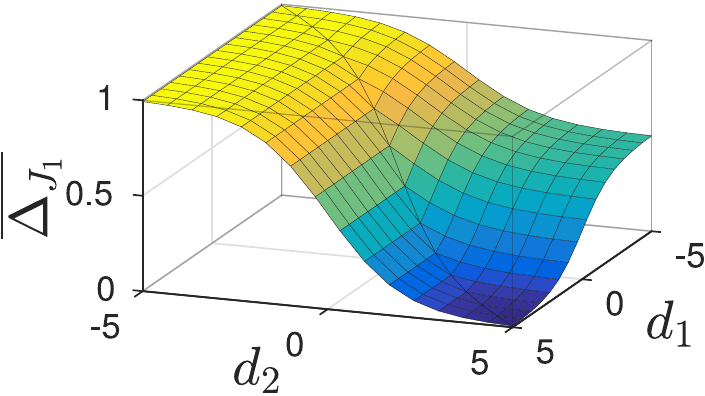}
    \caption{$\GT = [-1, 1]$}
  \end{subfigure}
 \begin{subfigure}{0.45\linewidth}    \centering\includegraphics[height=2cm]{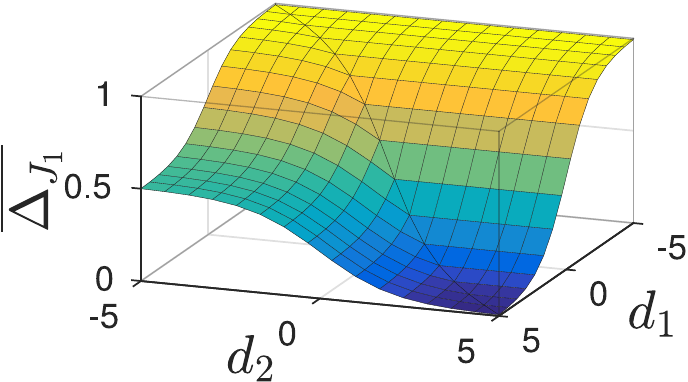}
    \caption{$\GT = [1, -1]$}
  \end{subfigure}
 \begin{subfigure}{0.45\linewidth}    \centering\includegraphics[height=2cm]{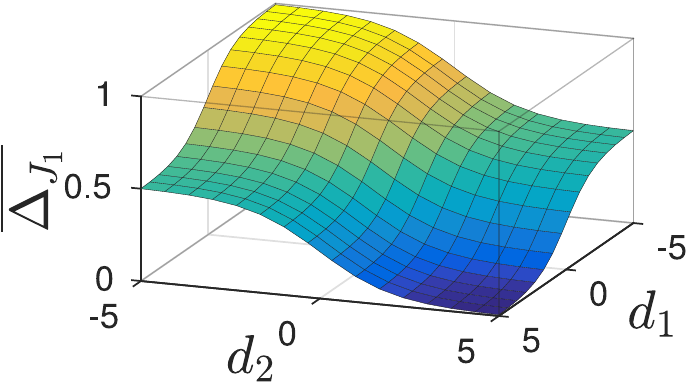}
    \caption{$\GT = [1, 1]$}
  \end{subfigure}%
\caption{Lovász-Softmax for the foreground class, with two classes $\{-1, 1\}$ and two pixels, for each ground truth labeling GT. 
The loss is plotted against the difference of unnormalized scores $d_i = F_i(y^*_i) - F_i(1 - y^*_i)$ for $i = 1, 2$.\label{fig:lovasz_softmax}
}
\end{figure}

In a segmentation setting with more than two classes, we propose a surrogate based on a logistic output instead of using a max-margin setting. 
Specifically we map the output scores of the model to probability distributions using a softmax unit as is done traditionally in the case of the cross-entropy loss.

We use the class probabilities $f_i(c) \in [0, 1]$ defined in Equation~\eqref{softmax} to construct a vector of pixel errors $\vec{m}(c)$ for class $c\in\mathcal{C}$  defined by
\begin{equation}
m_i(c) = \begin{cases}
        1 - f_i(c) & \text{if } c = y^*_i, \\
        f_i(c) & \text{ otherwise. } \\
      \end{cases}
\end{equation}
We use the vector of errors $\vec{m}(c) \in [0, 1]^p$ to construct the loss surrogate to $\J$, the Jaccard index for class $c$:
\begin{equation}\label{eq:lovasz1class}
\operatorname{loss}(\vec{f}(c)) = \ext{\J}(\vec{m}(c))
\end{equation}

When considering the class-averaged mIoU metric, common in semantic segmentation, we average the class-specific surrogates; hence we define the Lovász-Softmax loss as
\begin{equation}\label{eq:softmax-lovasz-average}
\operatorname{loss}(\vec{f}) = \frac{1}{|\mathcal{C}|} \sum_{c \in \mathcal{C}}\ext{\J}(\vec{m}(c))
\end{equation}
which is piecewise linear in $\vec{f}$, the normalized network outputs. Figure~\ref{fig:lovasz_softmax} show this loss as a function of the unnormalized vector outputs $\vec{F}$ for a prediction of two pixels. In the limit of large scores (confident outputs), the probability vectors at each pixel $(f_i(c))_{c\in\mathcal{C}}$ are close to an indicator vector, and we recover the values of the discrete Jaccard index for the corresponding discrete labeling with respect to the ground truth, as seen on the figure. 

\section{Optimization of intersection over union}

Naïve computation of the Lovász extension (Equation~\eqref{eq:lovasz_extension}) applied to $\J$ can be achieved by sorting the elements of $\margs$ in $\mathcal{O}(p \log p)$ time and doing $\mathcal{O}(p)$ calls to $\J$. However, if we compute $\J$ by Equation~\eqref{eq:JaccardIndex}, each call will cost $\mathcal{O}(n)$. As $\perm$ is known in advance, we may simply keep track of the cumulative number of false positives and negatives in $\{\perm_1, \ldots, \perm_i\}$ for increasing $i$ yielding an amortized $\mathcal{O}(1)$ cost per evaluation of $\J$ (cf.~\cite[Equation~(43)]{Yu2015b}). This computation also yields the gradient $\vec{g}(\vec{m})$ at the same computational cost. 
This is a powerful result implying that a tight surrogate function for the Jaccard loss is available and computable in time $\mathcal{O}(p \log p)$. 
The algorithm for computing the gradient of the loss surface resulting from this procedure is summarized in Algorithm~\ref{alg:Gamma}.  

\begin{algorithm}
\renewcommand{\algorithmicrequire}{\textbf{Inputs:}}
\renewcommand{\algorithmicensure}{\textbf{Output:}}
\caption{Gradient of the Jaccard loss extension $\ext{\J}$}\label{alg:Gamma}
\begin{algorithmic}[1]
\REQUIRE vector of errors $\vec{m}(c) \in \mathbb{R}_+^p$\\
class foreground pixels $\vec{\delta} = \{\gt = c\} \in \{0, 1\}^p$
\ENSURE $\vec{g}(\vec{m})$ gradient of $\ext{\J}$ (Equation~\eqref{eq:lovasz_derivative})
\STATE $\vec{\pi} \gets \text{decreasing sort permutation for } \vec{m}$
\STATE $\vec{\delta_\pi} \gets (\delta_{\pi_i})_{i\in [1, p]}$
\STATE $\text{\bfseries intersection} \gets \text{sum}(\vec{\delta}) - \text{\bfseries cumulative\_sum}(\vec{\delta_\pi})$
\STATE $\text{\bfseries union} \gets \text{sum}(\vec{\delta}) + \text{\bfseries cumulative\_sum}(1 - \vec{\delta_\pi})$
\STATE $\vec{g}  \gets 1 - \text{\bfseries intersection}/\text{\bfseries union}$
\IF{$p > 1$}
	\STATE $\vec{g}[2:p] \gets \vec{g}[2:p] - \vec{g}[1:p-1]$
\ENDIF 
\RETURN $\vec{g}_{\vec{\pi}^{-1}}$
\end{algorithmic}
\end{algorithm}

\subsection{Image--mIoU vs.\ dataset--mIoU}\label{sec:differentMeasures}
The official metric of the semantic segmentation task in Pascal VOC~\cite{everingham2010pascal} and numerous other popular competitions is the dataset--mIoU, 
\begin{equation}
\text{dataset--mIoU} = \frac{1}{|\mathcal{C}|}\sum_{c\in\mathcal{C}} J_c(\gt, \pred),
\end{equation}
where $\gt$ and $\pred$ contain the ground truth and predicted labels of~\emph{all pixels in the testing dataset}. 

The Lovász-Softmax loss considers an ensemble of pixel predictions for the computation of the surrogate to the Jaccard loss. 
In a stochastic gradient descent setting, only a small numbers of pixel predictions are taken into account in one optimization step. 
Therefore, the Lovász-Softmax loss cannot directly optimize the dataset--mIoU. 
We can compute this loss over individual images, optimizing for the expected image--mIoU, or over each minibatch, optimizing for the expected batch--mIoU. 
However, it is not true in general that
\begin{equation}
\mathbb{E}{\left(\frac{\text{intersection}}{\text{union}}\right)}
\approx
\frac{\mathbb{E}(\text{intersection})}{\mathbb{E}(\text{union})},
\end{equation}
and we found in our experiments that optimizing the image--mIoU or batch--mIoU generally degrades the dataset--mIoU compared with optimizing the standard cross-entropy loss.

The main difference between the dataset and image--mIoU measures resides in the absent classes. When the network wrongly predicts a single pixel belonging to a class that is absent from an image, the image intersection over union loss corresponding to that class changes from $0$ to $1$. By contrast, a single pixel misprediction does not substantially affect the dataset--mIoU metric. 

Given this insight, we propose as an heuristic for optimizing the dataset--mIoU to compute the batch Lovász-Softmax surrogate by taking the average in Equation~\eqref{eq:softmax-lovasz-average} only over the classes present in the batch's ground truth. 
As a result, the loss is more stable to single predictions in absent classes, mimicking the dataset--mIoU. 
As outlined in our experiments, the optimization of the Lovász-Softmax restricted to classes present in each batch, effectively translates into gains for the dataset--mIoU metric.

We propose an additional trick for the optimization of the dataset--mIoU. Since the mIoU gives equal importance to each class, and to make the expectation of the batch--mIoU closer to the dataset--mIoU, it seems important to ensure that we feed the network with samples from all classes during training. 
In order to enforce this requirement, we sample the patches from the training by cycling over every classes, such that each class is visited at least once every $|\mathcal{C}|$ patches. This method is referred to as \emph{equibatch} in our experiments.

\section{Experiments}
\subsection{Synthetic experiment}\label{sec:toy}

We demonstrate the relevance of using the Jaccard loss for binary segmentation with a synthetic binary image segmentation experiment. We generate $N = 10$ binary images of size $50 \times 50$ representing circles of various radius, and extract for each pixel $i$ a single feature using a unit variance Gaussian perturbation of the ground truth, $f_i  \sim \mathcal{N}(\epsilon,\,1)$ where $\epsilon = 1/2$ for the foreground and $-1/2$ for the background, as illustrated in Figure~\ref{toyfeats}.

\begin{figure}
  \centering
  \begin{subfigure}[t]{0.45\linewidth}    \centering\includegraphics[width=0.9\textwidth]{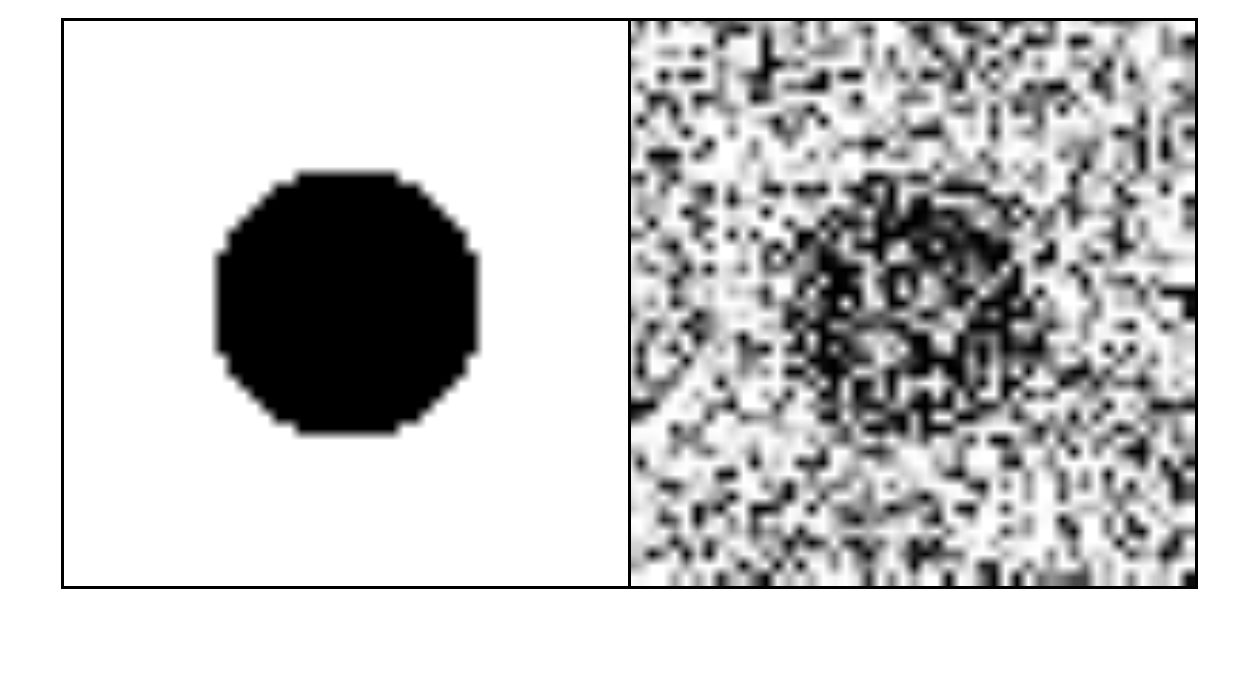}
    \caption{Sample label \& features\label{toyfeats}}
  \end{subfigure}\hfill
\begin{subfigure}[t]{0.55\linewidth}    \centering\includegraphics[width=1\textwidth]{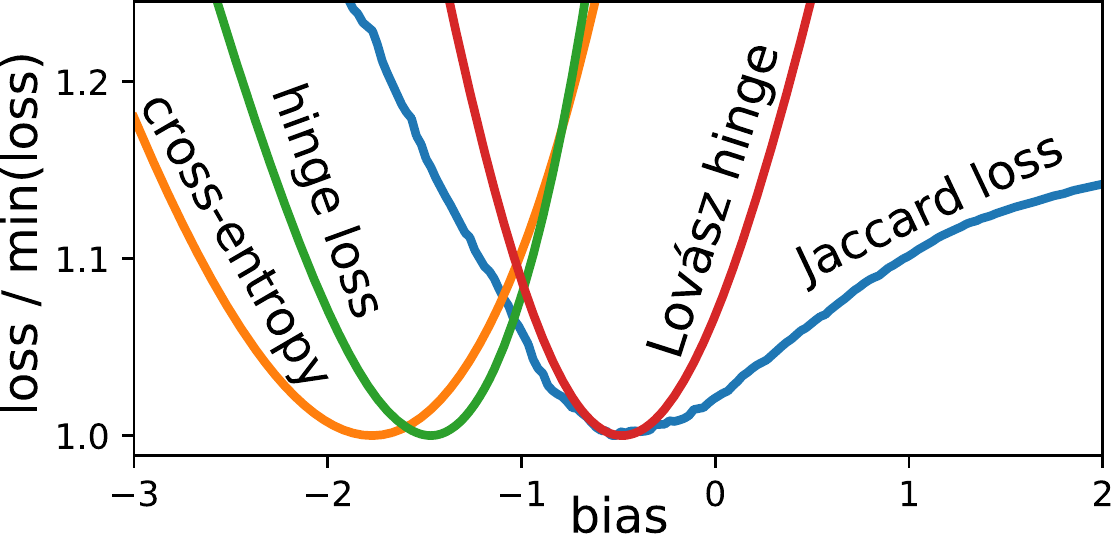}
    \caption{Relative losses for varying bias $b$\label{expbias}}
  \end{subfigure}
\caption{Synthetic model studied in \ref{sec:toy} and loss objectives.}
\end{figure}

We consider a model classifying pixels in the foreground class for $f_p > -b$, and we learn the bias term~$b$. An exhaustive search, illustrated in Figure~\ref{expbias}, shows that among the losses considered, only the Lovász hinge efficiently captures the absolute minimum of the Jaccard loss.

\subsection{Binary segmentation on Pascal VOC}
We base our Pascal VOC experiments on the DeeplabV2-single-scale semantic segmentation network~\cite{CP2016Deeplab}. The network uses a Resnet-101~\cite{he2016deep} based architecture, re-purposed for image segmentation, notably using dilated (or \emph{atrous}) convolutions. 
We use the initialization weights provided by the authors. These weights were pre-trained on MS-COCO~\cite{lin2014microsoft} using cross-entropy loss and weight decay. 
We further fine-tune these weights on a segmentation dataset consisting of Pascal VOC 2012 training images~\cite{pascal-voc-2012} and the extra images provided by~\cite{BharathICCV2011}, as is common in recent semantic image segmentation applications.

For our binary segmentation experiments, we perform an initial fine-tuning of the weights using cross-entropy loss alone jointly on the 21 classes of Pascal VOC (including the \texttt{background} class); this constitutes our basis network. We then turn to binary segmentation by selecting one particular class and finetune the output of the network for the selected class. 
In order to consider a realistic binary segmentation setting, for each class, we sample the validation set such that half of the images contain at least one foreground pixel. 
The training is done on random crops of size $321\times321$ extracted from the training set, with random scale and horizontal flipping. 
Training batches are randomly sampled from the training set such that half of the selected images contain the foreground class on average. 

Our experiments revolve around the choice of the training loss during fine-tuning to binary segmentation. We do a fine-tuning of $2$ epoch iterations, with an initial learning rate of $5 \cdot 10^{-4}$, reduced to $1 \cdot 10^{-4}$ after $1$ epoch.

\paragraph{Performance of the surrogate} Table~\ref{table:VOClossesSummary} shows the average of the losses considered after a training with different loss objectives. 
Evidently, training with a particular loss leads generally to a better objective value of this loss on the validation set.
Moreover, we see that the Lovász hinge acts as a good surrogate of the discrete image--IoU, leading to a better validation accuracy for this measure.

\begin{table}
\centering
\caption{Average of mean validation binary losses over the 20 Pascal VOC categories, after a training with cross-entropy, hinge, and Lovász hinge loss. The image--mIoU of the basis network, trained for all categories, is equal to 78.29.}
\label{table:VOClossesSummary}
\begin{tabular}{@{}lccc@{}}
\toprule
\emph{Training loss $\rightarrow$} & \multicolumn{1}{l}{Cross-entropy} & \multicolumn{1}{l}{Hinge} & \multicolumn{1}{l}{Lovász hinge} \\ \midrule
Cross-entropy & \textbf{6.84}                     & 6.96                      & 7.91                              \\
Hinge         & 7.81                              & \textbf{6.95}             & 7.11                              \\
Lovász hinge & 8.37                              & 7.45                      & \textbf{5.44}                     \\
Image--IoU   & 77.14                             & 75.8                      & \textbf{80.5}                     \\ \bottomrule
\end{tabular}
\end{table}

\begin{figure}
\centering
\includegraphics[width=\linewidth]{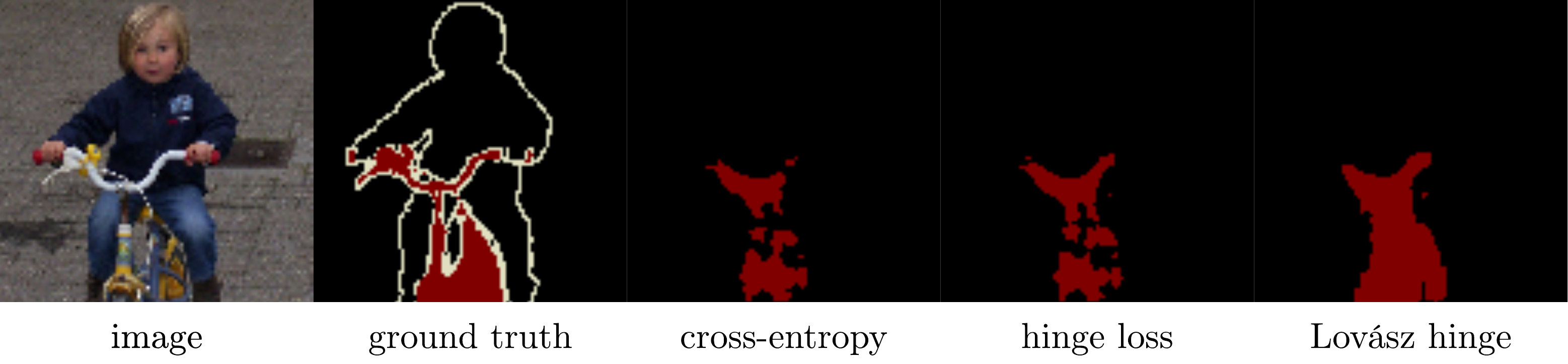}
\caption{Binary \texttt{bicycle} masks predicted on a validation image after training the network under various losses.\label{fig:binaryseg}}
\end{figure}
Figure~\ref{fig:binaryseg} shows example binary segmentation mask outputs. We notice that the Jaccard loss tends to fill gaps in segmentation, recover small objects, and lead to a more sensible segmentation globally, than other losses considered.

\paragraph{Comparison to prior work} \cite{rahman2016optimizing} propose separately approximating $I \simeq \sum_{i=1}^p F_i \, [y_i^* = 1]$ and $U \simeq \sum_{i=1}^n (p_i + [y_i^* = 1]) - I$ for optimization of binary $\text{IoU}\simeq I/U$. In our experiments, we were not able to observe a consistent improvement of the IoU using this surrogate, contrary to the Lovász hinge. Details on this comparison are included in the Supplementary Material, Section~A. 

\begin{figure*}[ht]
  \centering
  \begin{subfigure}[h]{0.195\linewidth}    \centering\includegraphics[height=7cm]{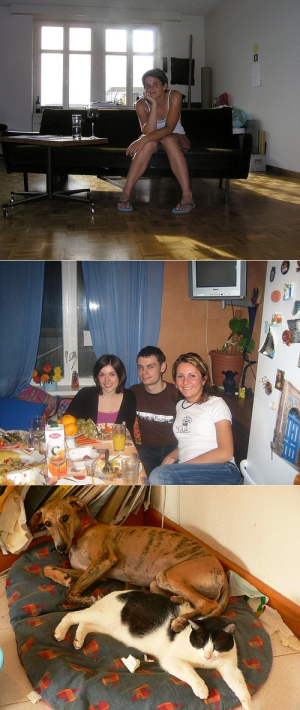}
    \caption{Input images}
  \end{subfigure}%
\begin{subfigure}[h]{0.195\linewidth} 
\centering\includegraphics[height=7cm]{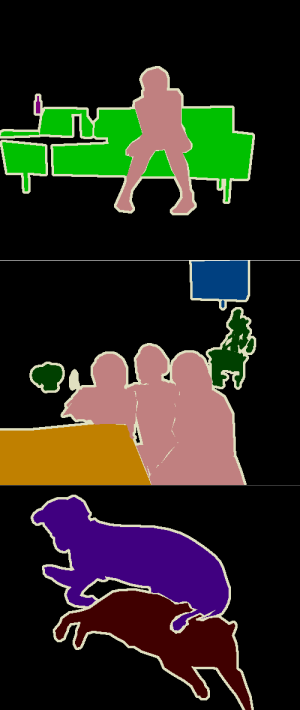}
    \caption{Ground truth masks}
  \end{subfigure}%
 \begin{subfigure}[h]{0.195\linewidth} \centering\includegraphics[height=7cm]{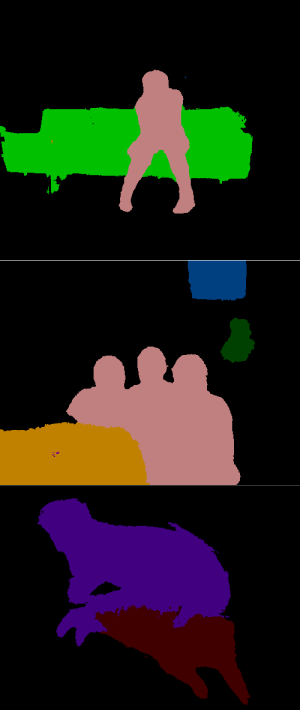}
    \caption{Lovász-Softmax + CRF}
  \end{subfigure}%
 \begin{subfigure}[h]{0.195\linewidth} \centering\includegraphics[height=7cm]{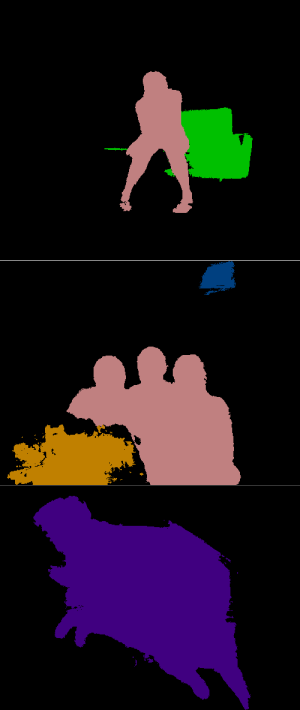}
    \caption{Cross-entropy + CRF}
  \end{subfigure}%
\caption{Multiclass segmentations after training with the Lovász-Softmax or the cross-entropy loss, and post-processed with Gaussian CRF. The color scheme follows the standard convention of the Pascal VOC dataset~\cite{pascal-voc-2012}.}\label{fig:ours-vs-xloss-samples}
\end{figure*}

\subsection{Multi-class segmentation on Pascal VOC}\label{sec:experimentsMultiClass}

We again useDeeplab-resnet-v2. This time, we exactly replicate the training procedure of the authors and following the same learning rate schedule, simply swapping out the loss for our multiclass surrogate, the Lovász-Softmax loss as described in Equation~\eqref{eq:softmax-lovasz-average}, with the mean being restricted to the classes present in a given batch.

As in the reference implementation, we use a stochastic gradient descent optimizer with momentum $0.9$ and weight decay $5\cdot 10^{-4}$; the learning rate at training iteration $k$ is
\begin{equation}\label{eq:lr-schedule}
\mathit{lr}^{(k)} = \mathit{lr}_\mathit{base}
\left(1 - \frac{k}{\mathit{max\_iter}}\right)^\mathit{power}
\end{equation}
where $\mathit{power}=0.9$ and $\mathit{lr}_\mathit{base} = 2.5 \cdot 10^{-4}$. 
We experiment either with $20\text{K}$ iterations of batches of size $10$ as in the reference paper, or with $30\text{K}$ iterations.
We train the network with patches of size $321 \times 321$, with random flipping and rescaling. 
The $1449$ validation images of Pascal VOC are included in the training only for experiments evaluated on the official test evaluation server. 

\begin{table*}[htb]
\centering
\caption{Performance of Deeplab-v2 single-scale trained with cross-entropy (x-loss) vs. Lovász-Softmax loss, for different network evaluations: raw single-scale network output, multi-scale, and Gaussian CRF post-processing.}
\label{results-multi}
\begin{tabular}{@{}lcccc@{}}
\toprule
                                  & \multicolumn{3}{c}{validation mIoU (\%)}          & test mIoU  (\%)                                                                               \\ \cmidrule(l){2-4} \cmidrule(l){5-5}
    & single-scale & multi-scale & multi-scale + CRF & multi-scale + CRF                                                                                 \\ \cmidrule(l){2-2} \cmidrule(l){3-3} \cmidrule(l){4-4} \cmidrule(l){5-5} 
x-loss                            & 74.64        & 76.23       & 76.53             & 76.44 
\\
x-loss + equibatch                & 75.53        & 76.70       & 77.31    & 78.05 
\\
x-loss + equibatch -- 30K iterations & 74.97        & 76.24       & 76.73             &                                                                                                   \\
\midrule
Lovász                            & 76.56        & 77.24       & 77.99             &                                                                                                   \\
Lovász + equibatch                & 76.53        & 77.28       & 78.49             &                                                                                                   \\
Lovász + equibatch -- 30K iterations & \textbf{77.41}        & \textbf{78.22}       & \textbf{79.12}    & \textbf{79.00} 
\\ \bottomrule
\end{tabular}
\end{table*}
\begin{table*}[htb]
\footnotesize
\centering
\caption{Per-class test IoU (\%) corresponding to the best-performing variants in Table~\ref{results-multi}.
}
\label{tab:perclass}
\resizebox{\textwidth}{!}{
\setlength\tabcolsep{1.5pt}
\begin{tabular}{@{}lcccccccccccccccccccc@{}}
\toprule
                            & airplane & cycle        & bird           & boat           & bottle         & bus            & car            & cat            & chair          & cow            & d. table   & dog            & horse          & mbike          & person         & plant          & sheep          & sofa           & train          & tv          \\ \midrule
x-loss                      & 92.95     & 41.06          & 87.06          & 61.23          & 77.6           & 91.99          & 88.11          & 92.45          & 32.84          & 82.48          & 59.6          & 90.13          & 89.83          & 86.77          & 85.79          & 58.06          & 85.31          & 52.00             & \textbf{84.47} & \textbf{71.26} \\
x-loss--equi.          & \textbf{93.32}     & 40.29          & \textbf{91.47} & 63.74          & 77.03          & 93.10           & \textbf{86.70}  & \textbf{93.37} & 34.79          & \textbf{87.92} & 69.74         & \textbf{89.53} & 90.61          & 84.70           & 85.13          & 59.23          & \textbf{87.71} & 64.46          & 82.89          & 68.57          \\
Lovász--equi 30K & 92.63     & \textbf{41.55} & 87.87          & \textbf{68.41} & \textbf{77.75} & \textbf{94.71} & \textbf{86.71} & 90.37          & \textbf{38.59} & 86.24          & \textbf{74.50} & 89.02          & \textbf{91.69} & \textbf{87.28} & \textbf{86.37} & \textbf{65.92} & 87.13          & \textbf{65.21} & 83.69          & 68.64          \\ \bottomrule
\end{tabular}
}
\end{table*}

We train Deeplab-resnet at a single input scale, which fits the memory constraints of a single GPU. 
We optionally evaluate the learned weights in a multiscale setting by taking the mean of the probabilities given by the network at scales $1$, $0.75$, and $0.5$, and also include the Gaussian CRF post-processing step used by Deeplab-v2. 
In this evaluation setting, we found that the baseline performance of the network trained with cross-entropy reaches $76.44$\% dataset--mIoU on the test set of Pascal VOC. 

Tables~\ref{results-multi} and~\ref{tab:perclass} present the scores obtained after training the network with cross-entropy or Lovász-Softmax loss, with and without \emph{equibatch}, under various evaluation regimes. For a given training and evaluation setting, our loss achieves higher mIoU. Figure~\ref{fig:ours-vs-xloss-samples} shows some example outputs.

Figure~\ref{fig:oursVSxlossIoU} shows the evolution of the validation mIoU over the course of the training. We notice that the performance gain manifests itself especially in the last epochs of the optimization.
Therefore, we also experiment with the same training setting with 30K iterations, to further benefit from the effects of the loss at these smaller learning rates. 
In agreement with our intuition, we see in Table~\ref{results-multi} that training with our surrogate benefits from a larger number of iterations, in contrast to the original training with cross-entropy.

\begin{figure*}[ht]
 \centering
  \begin{subfigure}[t]{0.32\linewidth}\centering
  \includegraphics[height=2.7cm]{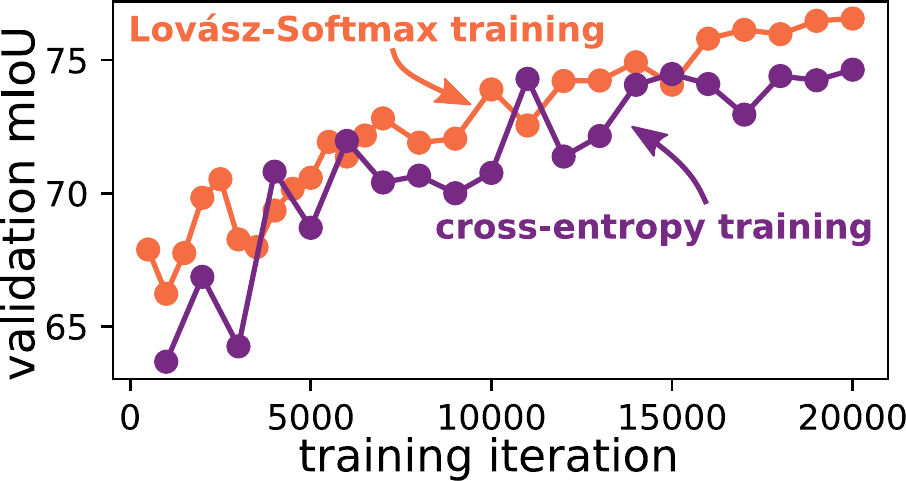}
\caption{Dataset mIoU on the validation set over the course of the Lovász-Softmax or cross-entropy optimization.\label{fig:oursVSxlossIoU}}
  \end{subfigure}
  \hfill
\begin{subfigure}[t]{0.32\linewidth}\centering
\includegraphics[height=2.7cm]{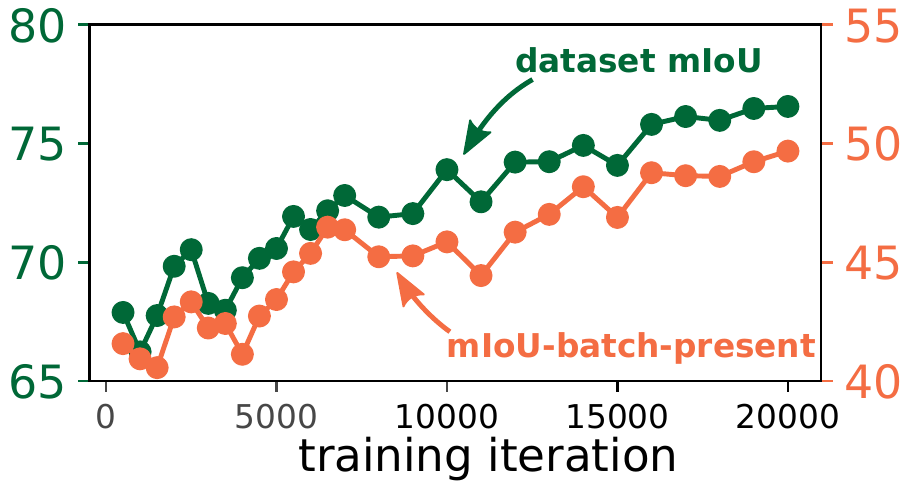}
\caption{Validation dataset--mIoU vs. batch--mIoU restricted to present classes during training with Lovász-Softmax.\label{fig:dataVSbatchIoU}}
  \end{subfigure}
  \hfill
 \begin{subfigure}[t]{0.32\linewidth}\centering
 \includegraphics[height=2.7cm]{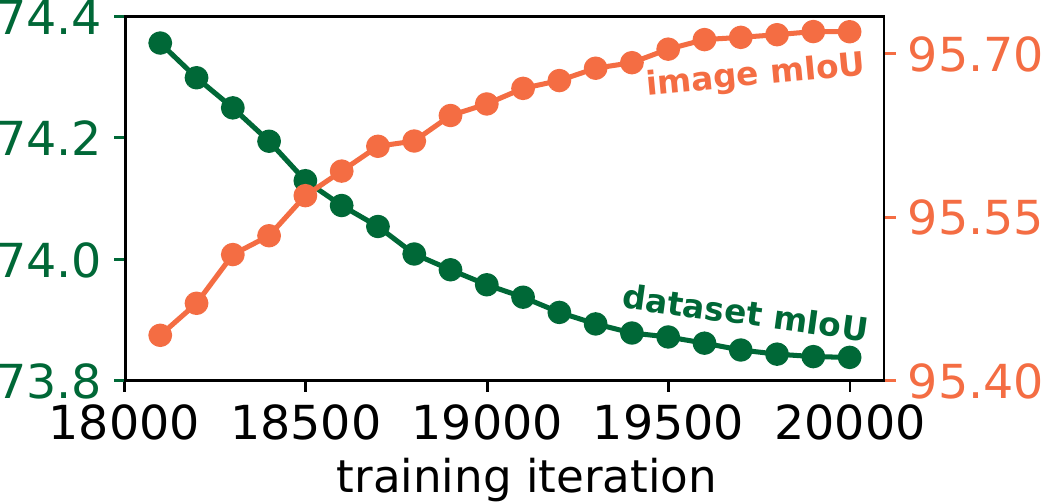}
    \caption{Validation dataset--mIoU and image--mIoU during training with Lovász-Softmax optimizing for image--mIoU.\label{fig:imageIoU}}
  \end{subfigure}
  \caption{Evolution of some validation measures over the course of the training.}
  \end{figure*}

The CRF post-processing step of Deeplab appears to bring complementary improvements to the use of our mIoU surrogate. While using \emph{equibatch} (batches with cyclic sampling from each class) does significantly help the cross-entropy loss with respect to the dataset--mIoU, its effect on the performance with Lovász-softmax seems marginal. This may be linked with the fact that our loss ignores classes absent from the minibatch ground truth, and therefore relies less on the order of appearance of the classes across batches. 
We found however that using \emph{equibatch} facilitates the convergence of the training, as it helps the network to consider all classes during the course of the optimization. This is especially important in the early stages of the optimization, where a class absent for too long can end up being dropped by the classifier in favor of the other classes.

\begin{figure}[ht]
  \centering
  \includegraphics[width=1.\linewidth]{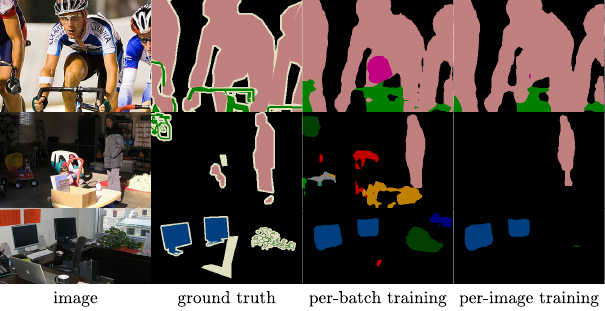}
\caption{Details of predicted masks after training with Lovász-Softmax per-batch vs. Lovász-Softmax per-image.}\label{fig:image-iou-samples}
\end{figure}
Figure~\ref{fig:dataVSbatchIoU} shows the joint evolution of the dataset--mIoU, and the batch--mIoU computed over present classes, during training. The correlation between these two measures justifies our choice of restricting the Lovász-Softmax to present classes as a proxy for optimizing the dataset--mIoU. 
As highlighted by Figure~\ref{fig:imageIoU}, the image--mIoU is a poor surrogate for the dataset--mIoU, as discussed in~Section~\ref{sec:differentMeasures}: optimizing one measure is generally detrimental to the other. 

Figure~\ref{fig:image-iou-samples} illustrates some qualitative differences between segmentations predicted by the network optimized for batch--mIoU and the network optimized for image--mIoU. The biggest difference between batch--mIoU and image--mIoU is the penalty associated with predicting a class that is absent from the ground truth. 
Accordingly, we notice that optimizing for image--mIoU tends to produce more sparse outputs, and output less extraneous classes, sometimes at the price of not including classes that are harder to detect. 

\paragraph{Comparison to prior work} 
Instead of changing the learning, Nowozin~\cite{Nowozin_2014_CVPR} designs a test-time decision function for mIoU based on the assumption of independent classifiers with calibrated probabilities. We applied this method on the Softmax output probabilities of the best model trained with cross-entropy loss (cross-entropy + \emph{equibatch}), and compare with the outputs from Lovász-Softmax (Lovász + equibatch $30\text{K}$). Since \cite{Nowozin_2014_CVPR} performs a local optimization (batches), we randomly select 20 batches of 21 images with every class represented, optimize the decision function, and compare the optimized mIoU of the batch with the mIoU of the selected batch in our output. The baseline has an average mIoU of $68.7 \pm 1.2$, our method significantly improves it to $72.5 \pm 1.2$, while~\cite{Nowozin_2014_CVPR} significantly degrades it to $65.1 \pm 1.4$. We believe this comes from the miscalibration of the neural network's probabilities, which adversely affects the assumptions of the decision function, as discussed in \cite[Sec. 5]{Nowozin_2014_CVPR}.

\subsection{Cityscapes segmentation with ENet}
We experiment with ENet, a segmentation architecture optimized for speed~\cite{paszke2016enet}, on the Cityscapes dataset~\cite{cordts2016cityscapes}. 
We fine-tune the weights provided by the authors, obtained after convergence of weighted cross-entropy loss, a loss that biases the cross-entropy loss to account for class inbalance in the training set. 
We do not need such a reweighing as our method inherently captures the class balancing of the mIoU. 

\begin{figure*}[ht]
  \centering
  \begin{subfigure}{0.33\linewidth}    \hfill\includegraphics[width=1\textwidth]{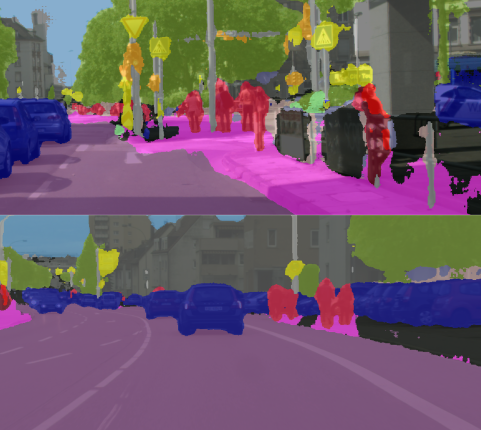}
    \caption{Initial ENet outputs~\cite{paszke2016enet}}
  \end{subfigure}
\begin{subfigure}{0.33\linewidth} 
\centering\includegraphics[width=1\textwidth]{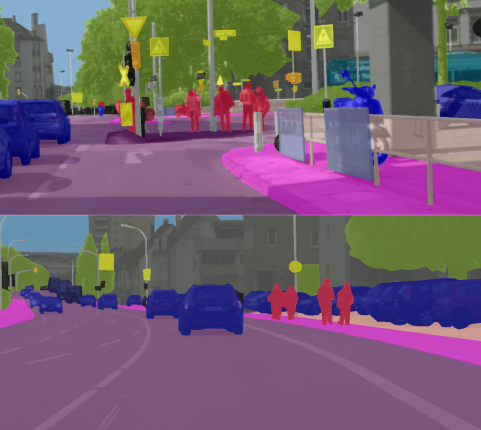}
    \caption{Ground truth masks}
  \end{subfigure}
 \begin{subfigure}{0.33\linewidth} \includegraphics[width=1\textwidth]{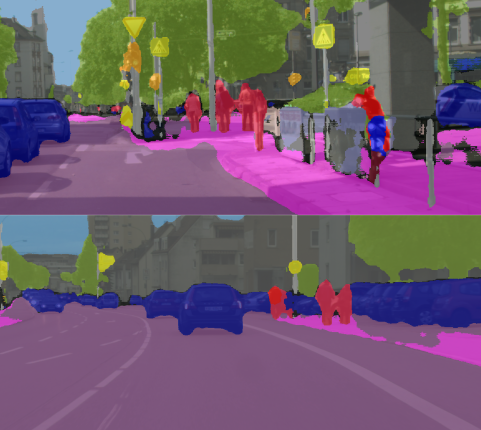}
    \caption{ENet + Lovász-Softmax fine-tuning}
  \end{subfigure}%
\caption{ENet: parts of output masks before and after fine-tuning with Lovász-Softmax (using the Cityscapes color palette).\label{fig:city-masks}}
\end{figure*}
\begin{figure}
\centering
\includegraphics[height=3cm]{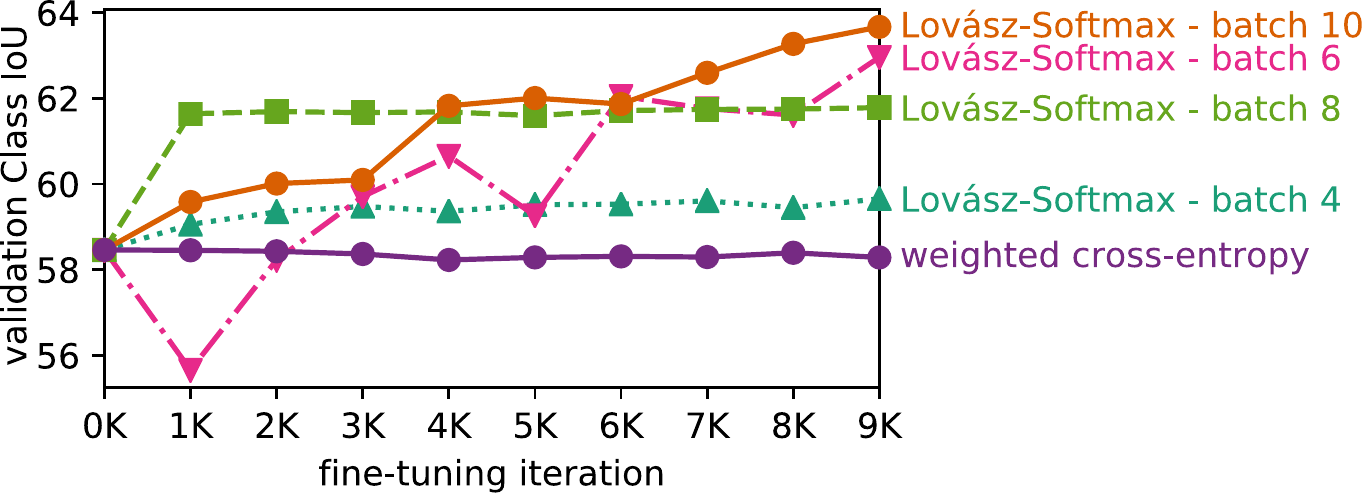}
\caption{Convergence of ENet on the validation set under fine-tuning with Lovász-Softmax, with various batch sizes.\label{fig:city-evo}}
\end{figure}
We finetune ENet using an Adam optimizer~\cite{kingma2014adam} with the same learning rate and schedule as in Equation~\eqref{eq:lr-schedule}. Consistent with \cite{paszke2016enet}, we use images of size $512\times 1024$ with no data augmentation. We snapshot every $1\text{K}$ iterations and report the test performance of snapshot $9\text{K}$ with batches of size 10, which corresponds to the highest validation score. 

Fig.~\ref{fig:city-evo} shows that our fine-tuning leads to a higher validation mIoU, while further training with weighted cross-entropy barely affects the performance -- as expected. Higher batch sizes generally lead to more improvement thanks to a better approximation of the dataset IoU. \emph{Equibatch} training did not make a difference in our experiments, which can be explained by the fact that the dataset is more uniform than Pascal VOC in terms of class representation. 
Note that we optimize for the mIoU measure, named \emph{Class IoU} in Cityscapes. Accordingly, we observe a substantial gain in performance in Cityscapes IoU metrics, with the Class IoU increasing from 58.29\% to 63.06\%. 
Reweighting the different classes in the average of the Lovász-Softmax loss (Equation~\eqref{eq:softmax-lovasz-average}) could allow us to target IoU-based measures which are weighted differently, such as CityScapes' iIoU metrics. 
Figure~\ref{fig:city-masks} presents some example output masks; we find that our fine-tuning generally reduces false positives and leads to finer details. 
Of course, our improved segmentation accuracy does not impact the high inference speed for which ENet is designed.

\begin{table}[ht]
\caption{Cityscapes results with Lovász-Softmax finetuning}
\label{tbl:city-enet}\vspace*{-2mm}
\centering\resizebox{.95\linewidth}{!}{%
\begin{tabular}{@{}lcccc@{}}
\toprule
               & Class IoU    & Class iIoU & Cat. IoU & Cat. iIoU \\ \midrule
ENet~\footnotemark       & 58.29          & 34.36        & 80.40          & \textbf{63.99}  %
\\
Finetuned~\footnotemark & \textbf{63.06} & 34.06        & \textbf{83.58} & 61.05 %
\\ \bottomrule
\end{tabular}
}\vspace*{-3mm}
\end{table}
\footnotetext[1]{\scriptsize\url{https://cityscapes-dataset.com/method-details/?submissionID=132}}
\footnotetext[2]{\scriptsize\url{https://cityscapes-dataset.com/method-details/?submissionID=993}}

\section{Discussion and Conclusions}\vspace*{-1mm}
In this work, we have demonstrated a versatile approach for optimizing the Jaccard loss for image segmentation. 
Our proposed method can be flexibly applied to a large number of function classes for segmentation, and we have demonstrated their effectiveness on state-of-the-art deep network architectures, substantially improving accuracies on semantic segmentation datasets simply by optimizing the correct loss during training. 
Qualitatively, we see greatly improved segmentation quality, in particular on small objects, while large objects tend to have consistent but smaller improvement in accuracy. 

This work shows that submodular measures such as the Jaccard index can be readily optimized in a continuous optimization setting. 
Further work includes the application of the approach to different tasks and losses exhibiting submodularity, and a derivation of specialized optimization routines given the piecewise-linear nature of the Lovász extension. 

The code associated with this publication, with replication of the experiments and implementations of the Lovász-Softmax loss, is released on~\url{https://github.com/bermanmaxim/LovaszSoftmax}.

\textbf{Acknowledgements. } This work is partially funded by Internal Funds KU Leuven and FP7-MC-CIG 334380.
We acknowledge support from the Research Foundation - Flanders (FWO) through project number G0A2716N. The authors thank J. Yu, X. Jia and Y. Huang for valuable comments and discussions.

\FloatBarrier
{\footnotesize
\bibliographystyle{ieee}
\bibliography{bermanbib,suppbib}
}

\clearpage
\counterwithin{figure}{section}
\counterwithin{table}{section}
\counterwithin{algorithm}{section}
\counterwithin{equation}{section}
\counterwithin{definition}{section}
\counterwithin{theorem}{section}
\counterwithin{proposition}{section}

\makeatletter
\renewcommand{\paragraph}{%
  \@startsection{paragraph}{4}%
  {\z@}{3.25ex \@plus 1ex \@minus .2ex}{-1em}%
  {\normalfont\normalsize\bfseries}%
}
\makeatother

\rhead{Supplementary Material}
\graphicspath{{supp-fig/}}
\pagenumbering{Roman}

\title{The Lovász-Softmax loss: A tractable surrogate for the optimization of the intersection-over-union measure in neural networks\\ \vspace{0.5cm}
\large Supplementary Material}

\author{Maxim Berman \quad Amal Rannen Triki \quad Matthew B. Blaschko\\
	Dept.\ ESAT, Center for Processing Speech and Images\\
 	KU Leuven, Belgium\\
{\tt\small \{maxim.berman,amal.rannen,matthew.blaschko\}@esat.kuleuven.be} \\
}

\maketitle

\appendix

\section{Detailed results for Section 4.2: binary segmentation on Pascal VOC}

\begin{figure*}[ht]
\centering
\includegraphics[width=0.95\textwidth]{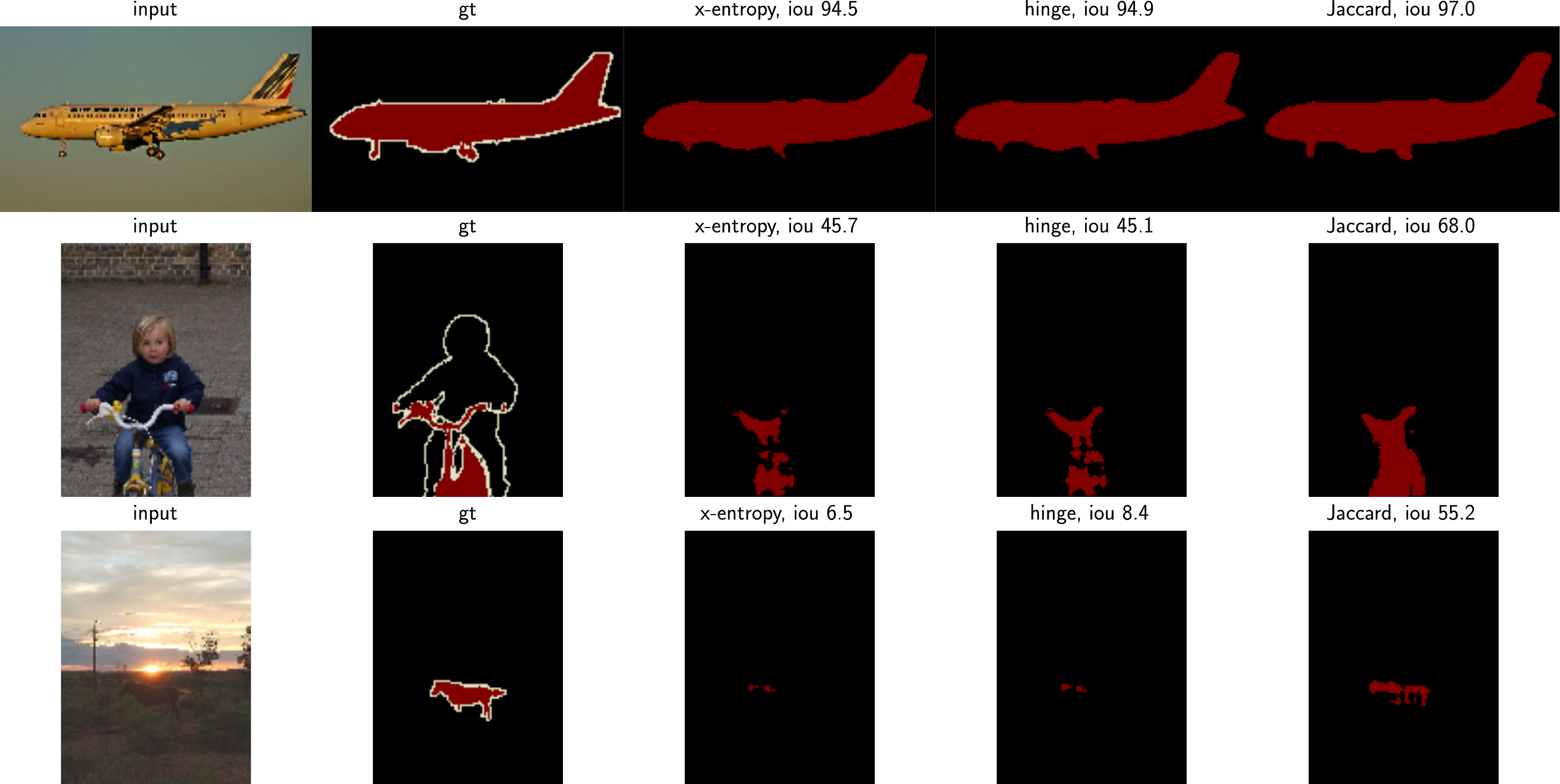}
\caption{Example binary segmentations trained with different losses and associated IoU scores on Pascal VOC.\label{fig:segmentations}}
\end{figure*}
Figure~\ref{fig:segmentations} shows segmentations obtained for binary foreground-background segmentation on Pascal VOC under different training losses, after finetuning a base multi-class classification network for a specific class. 
 We see that the Lovász hinge for the Jaccard loss tends to fill gaps in segmentation, recover small objects, and lead to a more sensible segmentation globally.

\begin{table*}[!ht]
\centering
\caption{Losses measured on our validation set of the 20 Pascal VOC categories, after a training with cross-entropy loss (\textbf{x}), hinge-loss (\textbf{h}), and Lovász-hinge (\textbf{j}). \textbf{b} indicates the performance of the base network, trained for all categories.}
\label{table:VOClosses}
\begin{flushright}
\resizebox{1.\textwidth}{!}{%
\begin{tabular}{@{}lcccccccccccccccc@{}}
 \toprule 
 & \multicolumn{4}{c}{aeroplane} & \multicolumn{4}{c}{bicycle} & \multicolumn{4}{c}{bird} & \multicolumn{4}{c}{boat}\\
\cmidrule(l){2-5}\cmidrule(l){6-9}\cmidrule(l){10-13}\cmidrule(l){14-17}
\multicolumn{1}{r}{\textit{training}} &\textbf{b} & \textbf{x} & \textbf{h} & \textbf{j} & \textbf{b} & \textbf{x} & \textbf{h} & \textbf{j} & \textbf{b} & \textbf{x} & \textbf{h} & \textbf{j} & \textbf{b} & \textbf{x} & \textbf{h} & \textbf{j}\\ \midrule
x-entropy, $\cdot 10^{-2}$ & & \textbf{2.8} & 3.3 & 4.1 &  & 12.3 & \textbf{11.0} & 11.3 &  & \textbf{3.4} & 4.0 & 4.6 &  & \textbf{6.4} & \textbf{6.4} & 6.9\\
hinge, $\cdot 10^{-2}$ & & 2.9 & \textbf{2.6} & 2.8 &  & 14.8 & 12.1 & \textbf{11.5} &  & 3.6 & 3.3 & \textbf{3.1} &  & 7.4 & \textbf{6.6} & \textbf{6.6}\\
Jacc-Hinge, $\cdot 10^{-1}$ & & 3.8 & 3.6 & \textbf{2.8} &  & 13.8 & 12.0 & \textbf{9.2} &  & 6.2 & 5.8 & \textbf{4.1} &  & 7.4 & 7.4 & \textbf{5.2}\\
Image-IoU, \% &86.2 & 88.6 & 87.7 & \textbf{89.6} & 63.2 & 61.2 & 58.7 & \textbf{66.3} & 84.5 & 82.1 & 81.3 & \textbf{86.9} & \textbf{80.3} & 75.8 & 73.2 & 79.9\\
\end{tabular}
} 
\resizebox{1.\textwidth}{!}{%
\begin{tabular}{@{}lcccccccccccccccc@{}}
 \toprule 
 & \multicolumn{4}{c}{bottle} & \multicolumn{4}{c}{bus} & \multicolumn{4}{c}{car} & \multicolumn{4}{c}{cat}\\
\cmidrule(l){2-5}\cmidrule(l){6-9}\cmidrule(l){10-13}\cmidrule(l){14-17}
\multicolumn{1}{r}{\textit{training}} &\textbf{b} & \textbf{x} & \textbf{h} & \textbf{j} & \textbf{b} & \textbf{x} & \textbf{h} & \textbf{j} & \textbf{b} & \textbf{x} & \textbf{h} & \textbf{j} & \textbf{b} & \textbf{x} & \textbf{h} & \textbf{j}\\ \midrule
x-entropy, $\cdot 10^{-2}$ & & \textbf{5.8} & \textbf{5.9} & 7.3 &  & \textbf{3.7} & 4.3 & 5.1 &  & \textbf{4.0} & 4.4 & 5.6 &  & \textbf{4.9} & 5.2 & 5.9\\
hinge, $\cdot 10^{-2}$ & & 6.6 & 5.6 & \textbf{4.5} &  & 3.9 & \textbf{3.4} & 3.9 &  & 4.4 & 4.0 & \textbf{3.5} &  & 5.4 & \textbf{4.9} & 5.1\\
Jacc-Hinge, $\cdot 10^{-1}$ & & 14.8 & 11.8 & \textbf{8.0} &  & 3.6 & 3.1 & \textbf{2.4} &  & 9.8 & 8.9 & \textbf{5.4} &  & 4.8 & 4.4 & \textbf{3.3}\\
Image-IoU, \% &\textbf{71.9} & 70.1 & 68.0 & 70.5 & 90.7 & 90.2 & 90.4 & \textbf{91.2} & 76.3 & 77.0 & 75.5 & \textbf{80.5} & 88.7 & 86.0 & 86.5 & \textbf{89.8}\\
\end{tabular}
} 
\resizebox{1.\textwidth}{!}{%
\begin{tabular}{@{}lcccccccccccccccc@{}}
 \toprule 
 & \multicolumn{4}{c}{chair} & \multicolumn{4}{c}{cow} & \multicolumn{4}{c}{diningtable} & \multicolumn{4}{c}{dog}\\
\cmidrule(l){2-5}\cmidrule(l){6-9}\cmidrule(l){10-13}\cmidrule(l){14-17}
\multicolumn{1}{r}{\textit{training}} &\textbf{b} & \textbf{x} & \textbf{h} & \textbf{j} & \textbf{b} & \textbf{x} & \textbf{h} & \textbf{j} & \textbf{b} & \textbf{x} & \textbf{h} & \textbf{j} & \textbf{b} & \textbf{x} & \textbf{h} & \textbf{j}\\ \midrule
x-entropy, $\cdot 10^{-2}$ & & 11.4 & \textbf{11.1} & 13.1 &  & \textbf{6.1} & 6.5 & 7.7 &  & 14.1 & \textbf{12.7} & 12.9 &  & \textbf{5.7} & 6.0 & 6.3\\
hinge, $\cdot 10^{-2}$ & & 13.3 & 11.8 & \textbf{11.0} &  & 6.9 & \textbf{6.2} & 7.6 &  & 16.7 & 14.5 & \textbf{13.7} &  & 6.3 & \textbf{5.8} & \textbf{5.8}\\
Jacc-Hinge, $\cdot 10^{-1}$ & & 16.6 & 14.4 & \textbf{9.8} &  & 5.6 & 5.1 & \textbf{4.1} &  & 12.5 & 10.7 & \textbf{7.9} &  & 5.6 & 5.0 & \textbf{3.4}\\
Image-IoU, \% &59.3 & 54.0 & 51.2 & \textbf{59.6} & 83.4 & 84.0 & 82.6 & \textbf{86.3} & 66.7 & 70.6 & 70.0 & \textbf{73.8} & 83.8 & 82.1 & 81.7 & \textbf{87.6}\\
\end{tabular}
} 
\resizebox{1.\textwidth}{!}{%
\begin{tabular}{@{}lcccccccccccccccc@{}}
 \toprule 
 & \multicolumn{4}{c}{horse} & \multicolumn{4}{c}{motorbike} & \multicolumn{4}{c}{person} & \multicolumn{4}{c}{potted-plant}\\
\cmidrule(l){2-5}\cmidrule(l){6-9}\cmidrule(l){10-13}\cmidrule(l){14-17}
\multicolumn{1}{r}{\textit{training}} &\textbf{b} & \textbf{x} & \textbf{h} & \textbf{j} & \textbf{b} & \textbf{x} & \textbf{h} & \textbf{j} & \textbf{b} & \textbf{x} & \textbf{h} & \textbf{j} & \textbf{b} & \textbf{x} & \textbf{h} & \textbf{j}\\ \midrule
x-entropy, $\cdot 10^{-2}$ & & \textbf{5.2} & 6.2 & 6.5 &  & \textbf{6.2} & 6.6 & 7.2 &  & \textbf{5.8} & \textbf{5.9} & 8.1 &  & \textbf{6.1} & 6.5 & 7.9\\
hinge, $\cdot 10^{-2}$ & & 5.7 & \textbf{5.3} & 5.8 &  & 7.0 & \textbf{6.4} & 6.8 &  & 6.5 & 6.0 & \textbf{5.4} &  & 6.9 & \textbf{6.1} & \textbf{6.1}\\
Jacc-Hinge, $\cdot 10^{-1}$ & & 6.0 & 5.7 & \textbf{4.6} &  & 5.1 & 4.8 & \textbf{3.7} &  & 8.1 & 7.4 & \textbf{4.9} &  & 12.4 & 10.4 & \textbf{8.2}\\
Image-IoU, \% &82.4 & 82.1 & 79.1 & \textbf{84.8} & 83.8 & 82.6 & 82.8 & \textbf{85.4} & 78.2 & 79.1 & 77.1 & \textbf{82.0} & 66.1 & 65.6 & 65.3 & \textbf{68.0}\\
\end{tabular}
} 
\resizebox{1.\textwidth}{!}{%
\begin{tabular}{@{}lcccccccccccccccc@{}}
 \toprule 
 & \multicolumn{4}{c}{sheep} & \multicolumn{4}{c}{sofa} & \multicolumn{4}{c}{train} & \multicolumn{4}{c}{tvmonitor}\\
\cmidrule(l){2-5}\cmidrule(l){6-9}\cmidrule(l){10-13}\cmidrule(l){14-17}
\multicolumn{1}{r}{\textit{training}} &\textbf{b} & \textbf{x} & \textbf{h} & \textbf{j} & \textbf{b} & \textbf{x} & \textbf{h} & \textbf{j} & \textbf{b} & \textbf{x} & \textbf{h} & \textbf{j} & \textbf{b} & \textbf{x} & \textbf{h} & \textbf{j}\\ \midrule
x-entropy, $\cdot 10^{-2}$ & & \textbf{6.4} & \textbf{6.5} & 7.8 &  & 13.8 & \textbf{13.4} & 14.9 &  & \textbf{7.0} & 7.2 & 8.8 &  & \textbf{5.6} & 6.0 & 6.2\\
hinge, $\cdot 10^{-2}$ & & 7.2 & \textbf{6.4} & 7.9 &  & 16.4 & \textbf{15.2} & 17.2 &  & 7.9 & \textbf{7.3} & 9.2 &  & 6.3 & 5.5 & \textbf{4.7}\\
Jacc-Hinge, $\cdot 10^{-1}$ & & 6.3 & 5.8 & \textbf{4.6} &  & 10.5 & 9.9 & \textbf{8.2} &  & 5.2 & 5.2 & \textbf{3.0} &  & 9.3 & 7.6 & \textbf{5.9}\\
Image-IoU, \% &83.7 & 80.3 & 78.1 & \textbf{85.3} & 69.7 & 69.6 & 67.7 & \textbf{72.1} & 88.8 & 83.9 & 81.3 & \textbf{89.7} & 78.1 & 77.8 & 77.8 & \textbf{80.6}\\
\bottomrule
\end{tabular}
} 
\end{flushright}

\end{table*}

\begin{figure}
  \centering
  \includegraphics[width=\columnwidth]{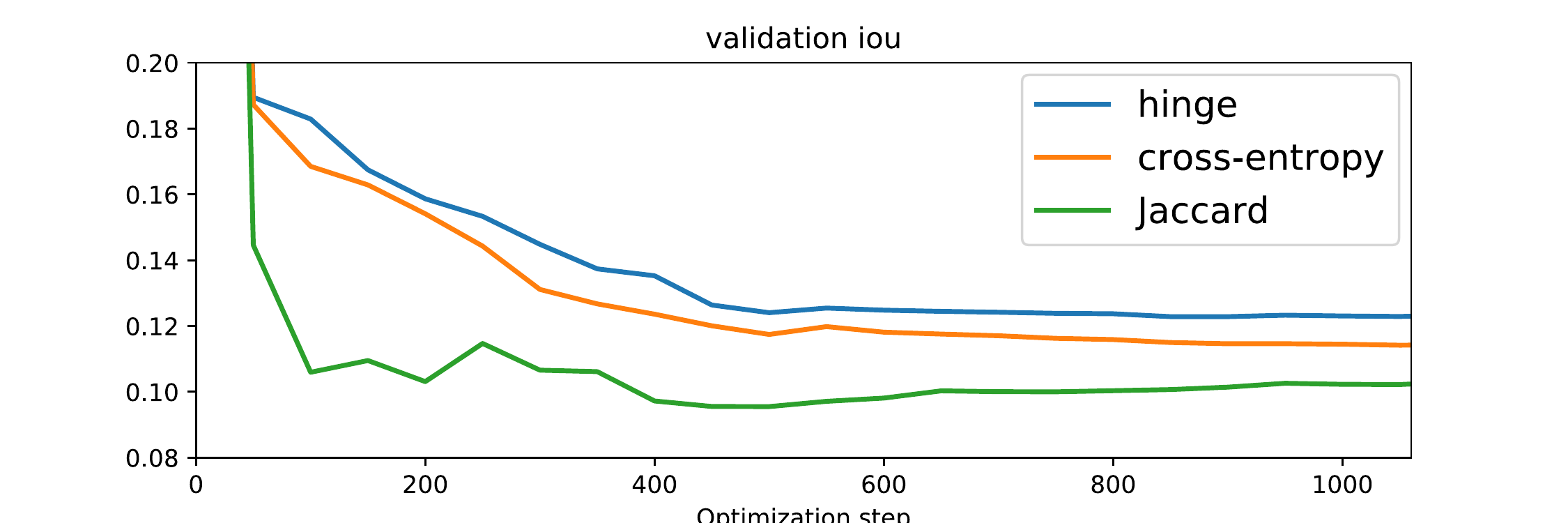}
\caption{Evolution of the validation IoU during the course of the optimization with the different losses considered.}\label{fig:IoUlossConvergenceLovaszCrossEntropyHinge}
\end{figure}

 Table~\ref{table:VOClosses} presents detailed scores for this binary segmentation task.
We notice a clear improvement of the per image-IoU by optimizing with the Jaccard loss. 
Moreover, the results are in agreement with the intuition that the best performance for a given loss on the validation set is achieved when training with that loss.
In some limited cases (\texttt{boat}, \texttt{bottle}) the performance of the base multi-class network is actually higher than the fine-tuned versions. 
Our understanding of this phenomenon is that the context is particularly important for these specific classes, and the absence of label for the other classes during finetuning impedes the predictive ability of the network. 
Additionally, Figure~\ref{fig:IoUlossConvergenceLovaszCrossEntropyHinge} presents an instance of convergence curves of this binary network, under the different losses considered. 

\paragraph{Comparison to prior work} \cite{rahman2016optimizing} propose separately approximating the intersection 
\begin{equation}
   I \simeq \sum_{i=1}^p F_i \, [y_i^* = 1],
\end{equation}
using the Iverson bracket notation, and the union
\begin{equation}
   U \simeq \sum_{i=1}^n (p_i + [y_i^* = 1]) - I
\end{equation}
for optimizing the binary $\text{IoU}\simeq I/U$. We compared the validation image mIoU under the loss of \cite{rahman2016optimizing} and the binary Lovász hinge, for all the categories of binarized Pascal VOC, in the setting of section 4.2. We chose for~\cite{rahman2016optimizing} the best-scoring among 3 learning rates.  As seen in Table~\ref{tab:compare_prior} the proxy loss in~\cite{rahman2016optimizing} does not reach the performance of our method. Since~\cite{rahman2016optimizing} uses the same approximation ``batch--IoU $\simeq$ dataset--IoU'', these observations extend to the binary dataset--IoU measure. 

\begin{table*}[htb]
\footnotesize
\centering
\caption{Per-class test IoU (\%) corresponding to the results by the best learning rate for~\cite{rahman2016optimizing} compared to the results of the  Lovász hinge.
}
\label{tab:compare_prior}
\resizebox{\textwidth}{!}{
\setlength\tabcolsep{1.5pt}
\begin{tabular}{@{}lcccccccccccccccccccc@{}}
\toprule
                            & airplane & cycle        & bird           & boat           & bottle         & bus            & car            & cat            & chair          & cow            & d. table   & dog            & horse          & mbike          & person         & plant          & sheep          & sofa           & train          & tv          \\ \midrule
\cite{rahman2016optimizing}                      &  79.9     &    54.7  &  75.5    &           72.5&   68.7   &   86.2   &    73.3       &   78.4   &     56.6      &   75.4        &  72.2   &  76.9 &   68.8   &  79.4    &   71.7   &    62.1  &    76.5      & 69.9 & 77.8 & 77.1 \\
Lovász-Hinge &   89.6   & 66.3 &   86.9& 79.9 & 70.5 & 91.2 & 80.5 &   89.8        & 59.6 &  86.3         & 73.8 &  87.6         & 84.8 & 85.4 & 82.0 & 68.0 &  85.3         & 72.1 & 89.7          &  80.6         \\ \bottomrule
\end{tabular}
}
\end{table*}

\section{Supplementary experiment: IBSR brain segmentation}

\paragraph{Data and Model} In order to test the Lovász-Softmax loss on a different type of images, we consider the publicly available dataset provided by the Internet Brain Segmentation Repository (IBSR)~\cite{rohlfing2012image}. This dataset is composed of Magnetic Resonance (MR) image data of 18 different patients annotated for segmentation. For this segmentation task, we used a model based on  Deeplab~\cite{CP2016Deeplab} adapted to IBSR by Shakeri et al.~\cite{shakeri2016sub}. Our evaluation follows the same procedure as in the cited paper: a subset of 8 subcortical structures is first selected: left and right thalamus, caudate, putamen, and pallidum, then 3 folds composed of respectively 11, 1, and 6 train, validation, and test volumes are used for training and testing. Table~\ref{tab:architecture}  details the model architecture to which we add batch normalization layers between the convolutional layers and their ReLU activations.

\paragraph{Settings} Similarly to~\cite{shakeri2016sub}, we consider the dataset composed of the 256 axial brain slices of each volume rather than using the 3D structure of the data. This dataset is composed of $256\times128$ grayscale images. Moreover, we discard the images that contain only the background class during training. For each fold, the training data is then limited to $\approx$ 800--900 slices. Training is done with stochastic gradient descent and a learning rate schedule to exponentially decrease from $10^{-1}$ to $10^{-3}$ over 35 epochs with either the cross-entropy loss as in the original model, or the Lovász-Softmax loss (the batch-mIoU for present classes variant). As we are interested on showing the effect of the loss, we do not apply the CRF post-processing proposed in ~\cite{shakeri2016sub}.

\paragraph{Results}
The mean Jaccard index and DICE over the 3 folds for each of the four classes (right + left) of interest along with the mean scores across all classes are given in Table~\ref{tab:BrainResults}, showing an improvement when using the Lovász-Softmax loss.  
Some qualitative results are shown in Figure~\ref{fig:BrainResults}, highlighting the improvements in detecting some fine subcortical structures when the Lovász-Softmax loss is used. 

\begin{table}[ht]
\centering
\resizebox{\linewidth}{!}{%
\begin{tabular}{@{}ccccccc@{}}
\toprule
Block & \multicolumn{3}{c}{convolution}     & \multicolumn{2}{c}{pooling} & batch norm. \\ \cmidrule(lr){2-4} \cmidrule(lr){5-6}
      & kernel      & \# filters & dilation & kernel          & stride    &         \\ \midrule
1     & $7\times 7$ & 64         & 1        & $3\times 3$     & 2         & yes         \\
2     & $5\times 5$ & 128        & 1        & $3\times 3$     & 2         & yes         \\
3     & $3\times 3$ & 256        & 2        & $3\times 3$     & 1         & yes         \\
4     & $3\times 3$ & 512        & 2        & $3\times 3$     & 1         & yes         \\
5     & $3\times 3$ & 512        & 2        & $3\times 3$     & 1         & yes         \\
6     & $4\times 4$ & 1024       & 4        & \multicolumn{2}{c}{none}    & yes         \\
7     & $1\times 1$ & 9          & 1        & \multicolumn{2}{c}{none}    & no          \\ \bottomrule
\end{tabular}
}
    \caption{Layers used for the brain image segmentation.}
    \label{tab:architecture}
\end{table}
\begin{table*}[ht]
\centering
\caption{Test results on IBSR brain segmentation task - Average on 3 folds }
\label{tab:BrainResults}
\begin{tabular}{@{}ccccccc@{}}
\cmidrule(l){3-7}
\multicolumn{1}{l}{}                                                      & \multicolumn{1}{l}{} & \begin{tabular}[c]{@{}c@{}}Thalamus\\ Proper\end{tabular} & Caudate & \multicolumn{1}{l}{Putamen} & Pallidum & Mean \\ \midrule
\multirow{2}{*}{\begin{tabular}[c]{@{}c@{}}Cross\\ Entropy\end{tabular}}  & Jaccard              & 72.74                                                      & 52.31    & 61.55                        & 54.04     & 60.16 \\ \cmidrule(l){2-7} 
                                                                          & DICE                 & 84.17                                                      & 68.33    & 76.07                        & 70.02     & 74.65 \\ \midrule
\multirow{2}{*}{\begin{tabular}[c]{@{}c@{}}Lovász\\ Softmax\end{tabular}} & Jaccard              & 73.56                                                      & 54.44    & 62.57                        & 55.74     & 61.55 \\ \cmidrule(l){2-7} 
                                                                          & DICE                 & 84.74                                                      & 70.25    & 76.89                        & 71.50     & 75.84 \\ \bottomrule
\end{tabular}
\end{table*}

\begin{figure*}[ht]
\centering
\includegraphics[width=0.7\textwidth]{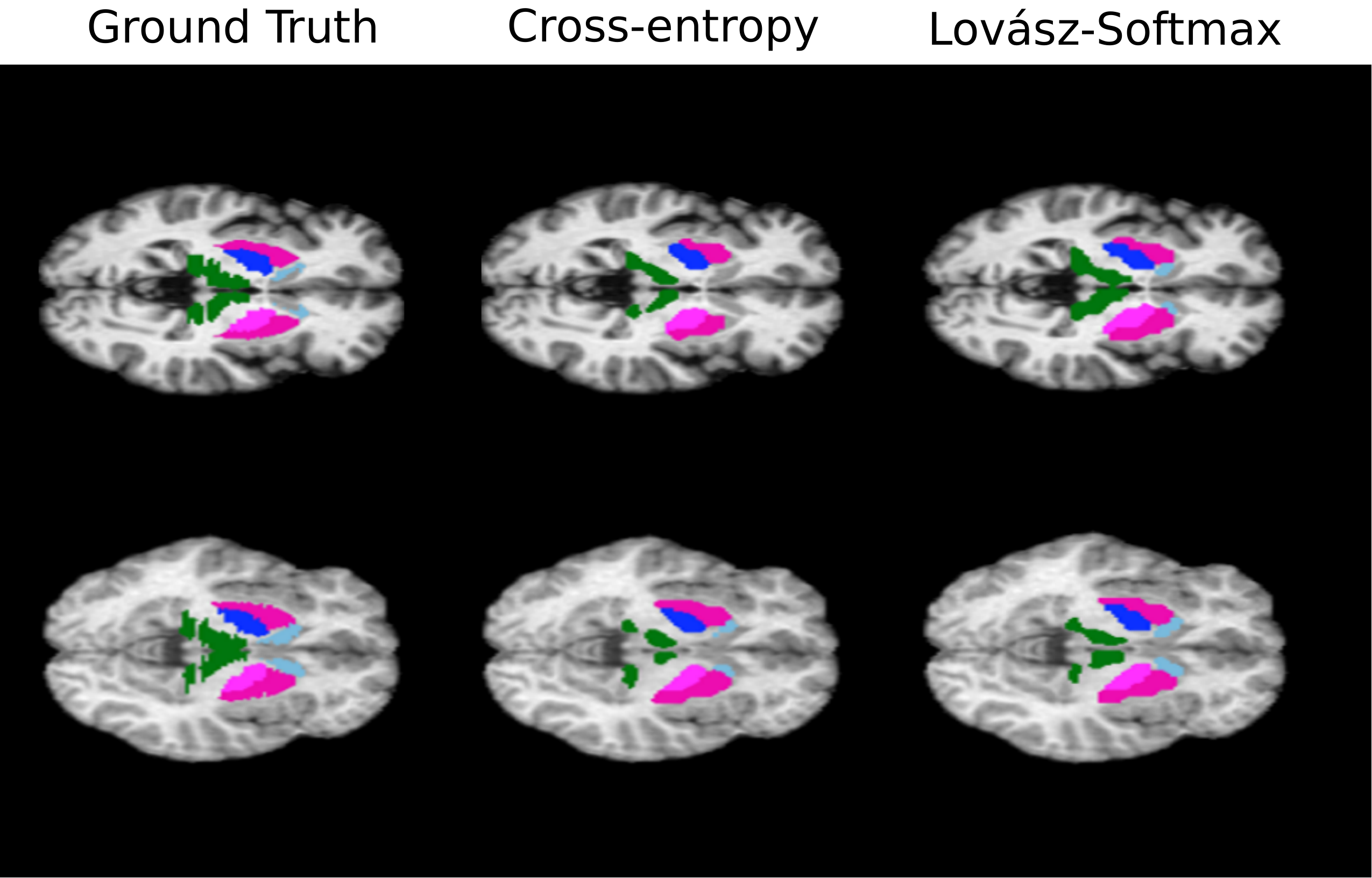}
\caption{Some examples of segmentation on the ISBR dataset. These examples are taken from two different patients and two different folds, and show an improvement in the segmentation of some fine structures when the Lovász-Softmax loss is used. } \label{fig:BrainResults}
\end{figure*}

\section{Proximal gradient algorithm}
We have developed a specialized optimization procedure for the Lovász Hinge for binary classification with the Jaccard loss, based on a computation of the proximal operator of the Lovász Hinge.
We include this algorithm here for completeness but have not used it for the main results of the paper, instead relying on standard stochastic gradient descent with momentum. 
The \emph{proximal gradient algorithm} we propose here has been independently proposed by Frerix et al.~\cite{frerix2018proximal}.

Our motivation for the proximal gradient algorithm stems from the piecewise-linearity of our loss function, which might destabilize stochastic gradient descent. Instead we would like to exploit the geometry of the Lovász Hinge. 
We therefore analyze the applicability of (variants of) the proximal gradient algorithm for optimization of a risk functional based on the Lov\'{a}sz hinge.
\begin{definition}[Proximal operator]
The proximal operator of a function $f$ with a regularization parameter $\lambda$ is 
\begin{equation}\label{prox_alg}
\prox_{f, \lambda}(x) = \argmin_{u} f(u) + \frac{\lambda}{2} \| u - x \|^2
\end{equation}
\end{definition}

We consider the problem of minimizing a (sub)differentiable function $f$. Iterative application of the proximal operator with an appropriately decreasing schedule of $\{\lambda_t\}_{0\leq t \leq \infty}$ leads to convergence to a local minimum analogously to gradient descent.  Furthermore, it is straightforward to show that, given an appropriately chosen schedule of $\lambda$ parameters, the proximal gradient algorithm will converge at least as fast as gradient descent.

\begin{proposition}
Given a gradient descent parameter $\eta$, $x_{t+1} = x_{t} - \eta \nabla f(x_t)$, there exists a set of descent parameters $\{\lambda_t\}_{0 \leq t \leq \infty}$ such that 
\begin{enumerate*}[label=(\roman*)]
\item the step size of the proximal operator is equivalent to gradient descent and 
\item $\operatorname{prox}_{f, \lambda_t}(x_t) \leq x_{t} - \eta \nabla f(x_t)$
\end{enumerate*}.
\end{proposition}
\begin{proof}
Starting with claim (i), we note that the proximal operator is the Lagrangian of the constrained optimization problem $\arg\min_{u} f(u)$ s.t.\ $\|x - u \|^2 \leq R$ for some $R > 0$, and we may therefore consider $\lambda_{t}$ such that $R_t = \| \eta \nabla f(x_t) \|^2$, where $\{x_t\}_{0 \leq t \leq \infty}$ is the sequence of values visited in gradient descent.

Claim (ii) follows directly from the definition of the proximal operator as the minimization of $f(u)$ within a ball of radius $R_t$ around $x_t$ must be at least as small as the value at the gradient descent direction.
\end{proof}
It is straightforward to convert a gradient descent step size schedule to an equivalent proximal gradient schedule of $\lambda_t$ values such that, were the objective linear, the two algorithms would be equivalent.  Indeed, the proximal gradient algorithm applied to a piecewise linear objective only differs from gradient descent at the boundaries between linear pieces, in which case it converges in a strictly smaller number of steps than gradient descent.

We optimize a deep neural network architecture by a modified backpropagation algorithm in which the gradient direction with respect to the loss layer is given by the direction of the empirical difference $x_t - \operatorname{prox}_{f}(x_t)$.  We note that this modification to the standard gradient computation is compatible with popular optimization strategies such as Adam~\cite{kingma2014adam}.  In initial experiments using the true gradient rather than that based on the proximal operator, we found that the use of momentum led to faster empirical convergence than Adam, and we therefore have based our subsequent comparison and empirical results on optimization with momentum. 

We show here that these momentum terms still do not lead in practice to as efficient update directions as those defined by the proximal operator.

\begin{definition}[Momentum~\cite{Sutskever2013IIM}]
Gradient descent with momentum is achieved with the following update rules
\begin{align}
v_{t+1} =& \alpha v_t + \nabla f(x_t)\\
x_{t+1} =& x_t - \eta v_{t+1} ,
\end{align}
where $\eta$ is the gradient descent parameter and $\alpha \in [0,1]$ is the momentum coefficient.
\end{definition}

Unrolling this recursion shows that momentum gives an exponentially decaying weighted average of previous gradient values, and setting $\alpha=0$ recovers classical gradient descent. 

Figure~\ref{fig:ProxVsMomentum} shows the behavior of gradient descent with momentum on the problem 
\begin{equation}\label{eq:examplePiecewiseLinearOptimization}
\min_{x\in \mathbb{R}^2} \max\left(0,\left\langle x, \begin{pmatrix} \nu \\ 0 \end{pmatrix} \right\rangle, \left\langle x, \begin{pmatrix} 0\\ 1 \end{pmatrix} \right\rangle \right),
\end{equation}
where $\nu$ is a positive scalar that allows us to adjust the relative scale of the gradients on either side of the boundary between the pieces.
In all cases, the momentum oscillates around piecewise-linear edges, and in Figure~\ref{fig:ProxVsMomentumNu1_3}, we see that traversing to a piece of the loss surface with very different slope can lead to multiple steps away from the boundary before returning to a steeper descent direction.  By contrast, the proximal algorithm immediately determines the optimal descent direction.

\begin{figure}
 \centering
  \begin{subfigure}{0.49\linewidth}    \centering\includegraphics[width=1\textwidth]{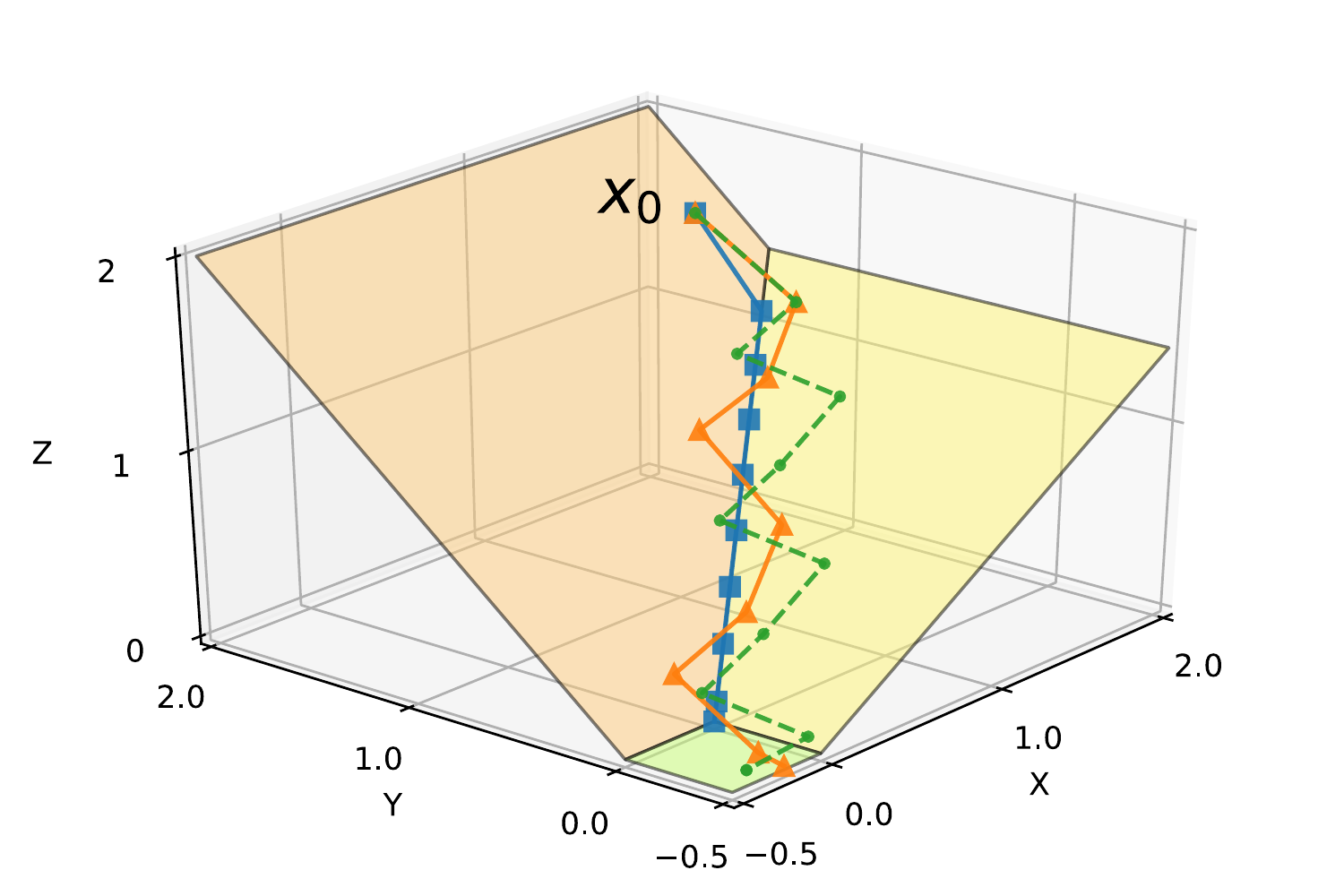}
    \caption{$\nu = 0.7$}
  \end{subfigure}
\begin{subfigure}{0.49\linewidth}    \centering\includegraphics[width=1\textwidth]{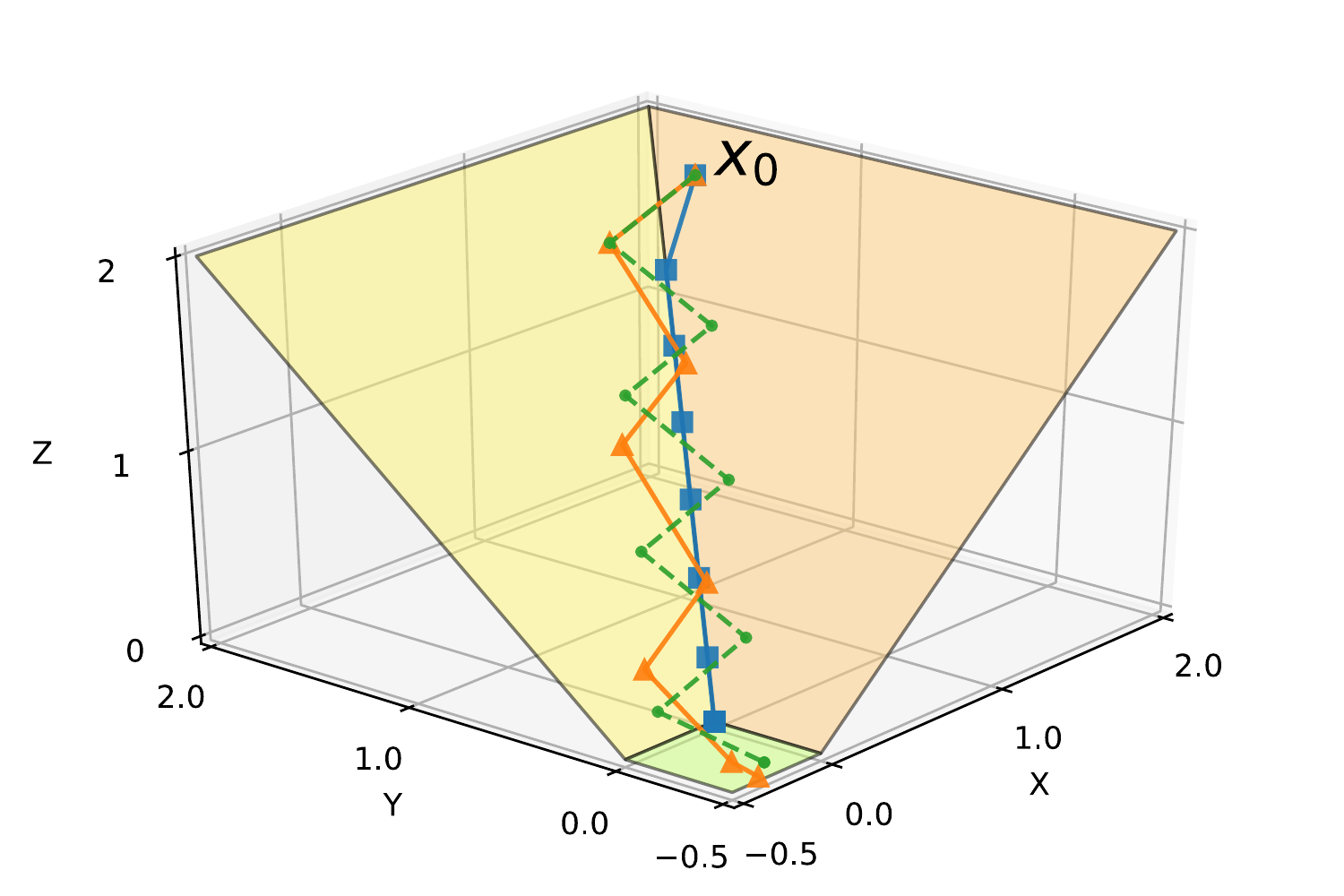}
    \caption{$\nu = 1.0$}
  \end{subfigure}
 \begin{subfigure}{0.49\linewidth}    \centering\includegraphics[width=1\textwidth]{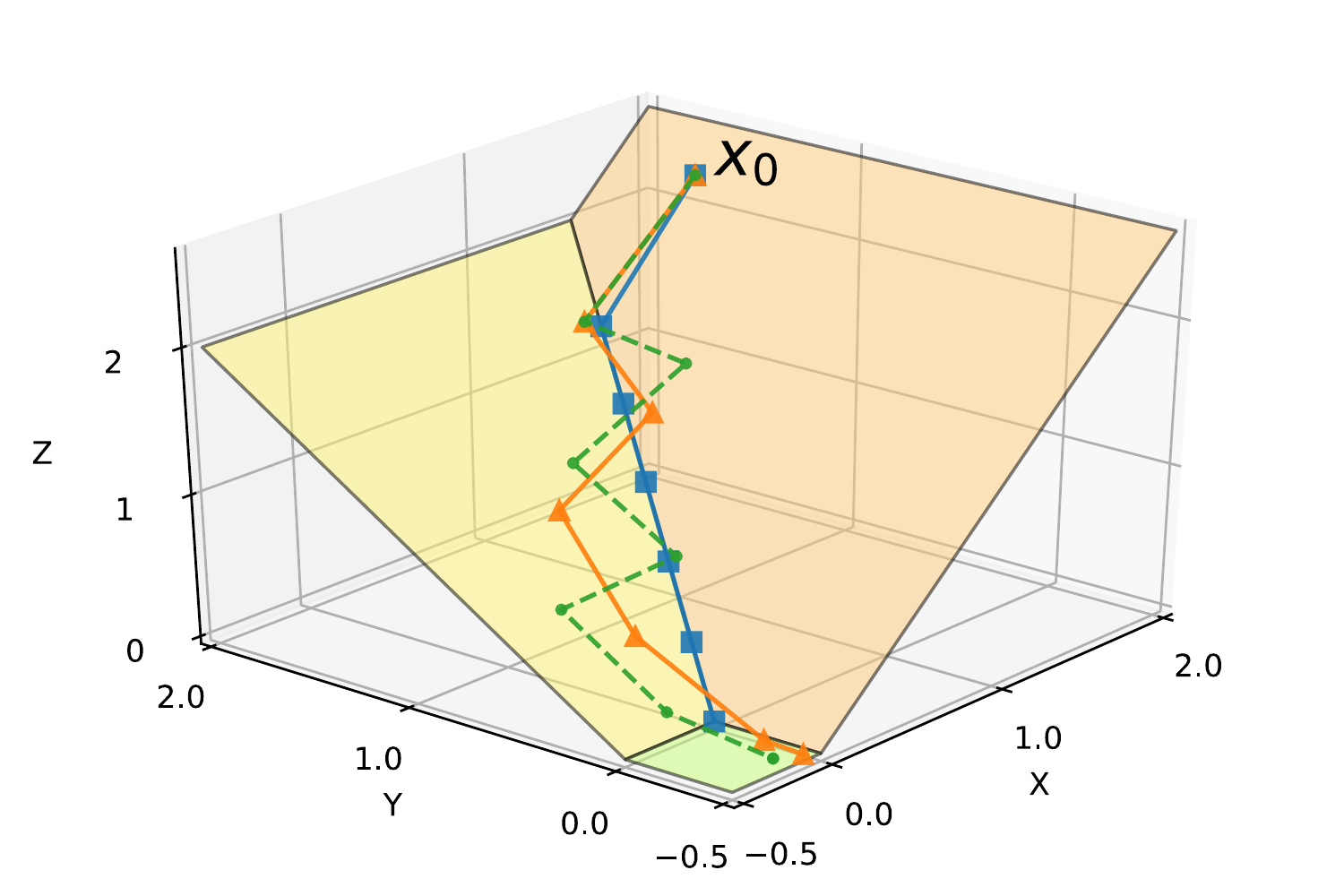}
    \caption{$\nu = 1.3$ \label{fig:ProxVsMomentumNu1_3}}
  \end{subfigure}
\caption{Optimization behavior of the piecewise-linear surface defined in Equation~\ref{eq:examplePiecewiseLinearOptimization}: gradient descent (green, dashed) and momentum (orange, plain) oscillate around the edge, while the proximal algorithm (green) finds the optimal descent direction.}\label{fig:ProxVsMomentum}
\end{figure}

\paragraph{Optimization study} We specialize the proximal gradient algorithm to our proposed Jaccard Hinge loss. We compute an approximate value of the proximal point to any initial point on the loss surface by following a greedy minimization path to the proximal objective~\ref{prox_alg}. This computation is detailed in Algorithm~\ref{alg:ProxLovaszJaccardSingle}.

\begin{algorithm}
\renewcommand{\algorithmicrequire}{\textbf{Input:}}
\renewcommand{\algorithmicensure}{\textbf{Output:}}
\caption{Computation of $\prox_{\ext{\Jf}, \lambda}(\margs)$}\label{alg:ProxLovaszJaccardSingle}
\begin{algorithmic}[1]
\REQUIRE Current $\margs$, $\ext{\Jf}$, $\lambda$
\ENSURE $\margs^* = \prox_{\Fe, \lambda}(\margs)$
\STATE $\vec{v}^0,\, \perm \gets \text{decreasing ordering of $\margs$ and permutation}$ \label{algLine:sortGamma}
\STATE $\vec{v} \gets \vec{v}^0$
\STATE $\g \gets \grad_{\vec{v}}{\Fe} \text{ (as a function of the sorted margins)} $ \label{algLine:GradientComputation}
\STATE $\vec{E} \gets \{\text{constraint } g_i = g_{i+1} = \ldots = g_{i + p}$\\ \hfill$\text{ for each equality } v_i = v_{i+1} = \ldots = v_{i + p} \}
$ \label{algLine:EqualityConstraints}
\STATE $\vec{c}_z \gets \text{constraint } g_{z + 1} = \ldots = g_d$\\ \hfill$\text{ for $z$ minimal index such that $v_{z} < 0$ }$ \label{algLine:ZeroThresholdConstraints}
\STATE $\text{finished} \gets \text{False}$
\WHILE{not finished}\label{algLine:mainWhileLoopProxSingle}
\STATE \textbf{if} $\g = \vec{0}:$ \textbf{break}
\STATE $\g \gets \proj_{\vec{E} \cup \{\vec{c}_z\}}{\g}$ \label{algLine:ProjectionToEqualityAndZeroConstraints}
\STATE $\vec{v}_\text{next} \gets$ projection of $\vec{v}$ on the closest edge of $\Fe$ in the direction $\g$ \label{algLine:NextProjectionOntoEdge}
\STATE $\text{stop} \gets 1/\lambda + \langle \vec{v} - \vec{v}^0, \g \rangle / \langle \g, \g \rangle$
\IF{$\text{stop} < \|\vec{v}_\text{next} - \vec{v}\|$}
  \STATE $\vec{v} \gets \vec{v} + \text{stop} \cdot \g$
  \STATE $\text{finished} \gets \text{True}$
\ELSE
  \STATE{ $\vec{v} \gets \vec{v}_\text{next}$ }
  \STATE{ Add the new constraint to $\vec{E}$ or update $\vec{c}_z$ }
\ENDIF
\ENDWHILE
\RETURN $\margs^* = \vec{v}_{\pi^{-1}}$
\end{algorithmic}
\end{algorithm}

\begin{figure}[ht]
  \centering
  \includegraphics[width=\linewidth]{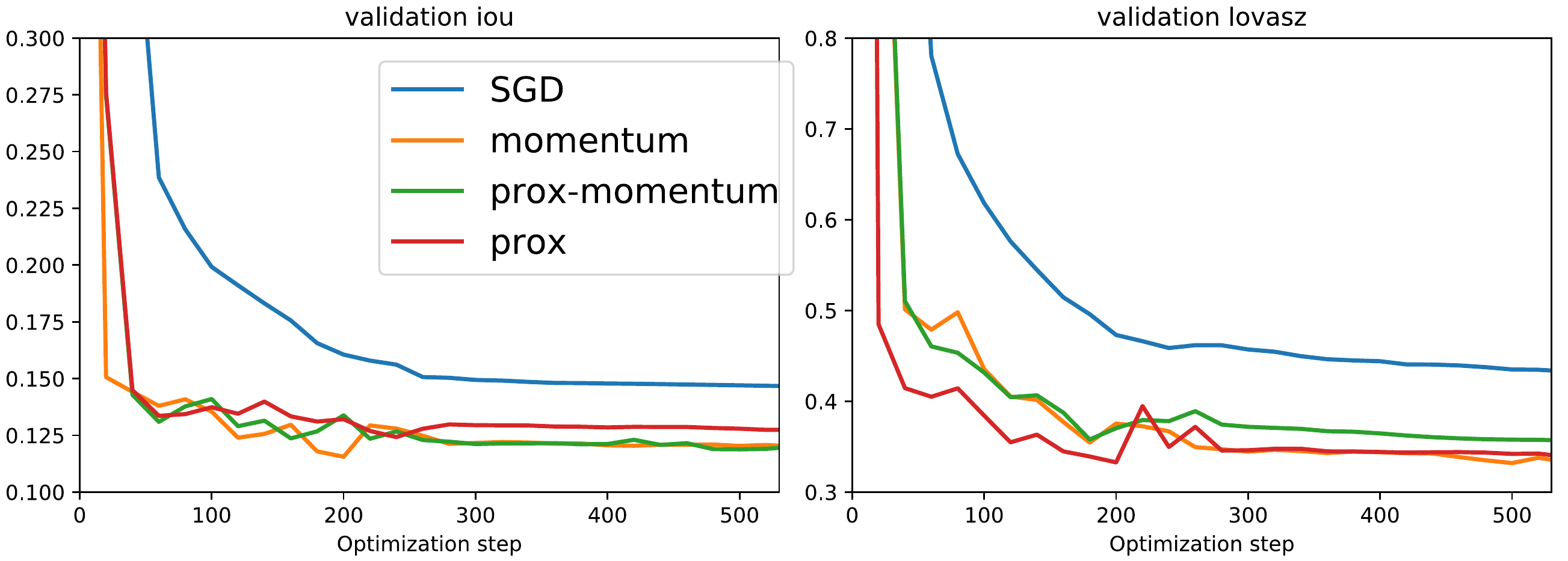}
\caption{Jaccard loss optimization with different optimization methods.}\label{fig:Loss OptimizationLastLayer}
\end{figure}

We investigate the choice of the optimization in terms of empirical convergence rates on the validation data.  We evaluate the use of varying optimization strategies for the last layer of the network in Figure~\ref{fig:Loss OptimizationLastLayer}. Experimentally, we find that the proximal gradient algorithm converges better than stochastic gradient descent alone, and has similar or better performance to stochastic gradient descent with momentum, which it can easily be combined with.

\end{document}

\end{document}